\documentclass[11pt]{article}
\usepackage{mainsty}
\usepackage[utf8]{inputenc}
\usepackage{authblk}
\usepackage{mathtools}
\usepackage{minted}
\usepackage{bm}
\usetikzlibrary{positioning}
\usetikzlibrary{backgrounds}
\usepackage{upref}

\usepackage{xcolor}
\newcommand{\tS}{\textup{\texttt{S}}}
\newcommand{\tP}{\textup{\texttt{P}}}
\newcommand{\tC}{\textup{\texttt{C}}}
\newcommand{\tO}{\textup{\texttt{O}}}
\newcommand{\bx}{\mathbf{x}}
\newcommand{\tV}{\textup{\texttt{V}}}

\newcommand{\edits}[1]{{\color{black} #1}}

\DeclareMathOperator{\relu}{\sigma}
\DeclareMathOperator{\atomic}{atom}

\usepackage{tikz}
\usepackage{etoolbox} 
\usepackage{listofitems} 

\tikzset{>=latex} 
\colorlet{myred}{red!80!black}
\colorlet{myblue}{blue!80!black}
\colorlet{mygreen}{green!60!black}
\colorlet{mydarkred}{myred!40!black}
\colorlet{mydarkblue}{myblue!40!black}
\colorlet{mydarkgreen}{mygreen!40!black}
\tikzstyle{node}=[very thick,circle,draw=myblue,minimum size=22,inner sep=0.5,outer sep=0.6]
\tikzstyle{connect}=[->,thick,mydarkblue,shorten >=1]
\tikzset{ 
  node 1/.style={node,mydarkgreen,draw=mygreen,fill=mygreen!25},
  node 2/.style={node,mydarkblue,draw=myblue,fill=myblue!20},
  node 3/.style={node,mydarkred,draw=myred,fill=myred!20},
}

\title{Neural Networks Generalize on Low Complexity Data}
\author{Sourav Chatterjee\footnote{Department of Statistics, Stanford University; Email: \href{mailto:souravc@stanford.edu}{\texttt{souravc@stanford.edu}}}, Timothy Sudijono\footnote{Department of Statistics, Stanford University; Email: \href{mailto:tsudijon@stanford.edu}{\texttt{tsudijon@stanford.edu}} }}
\date{\today}

\begin{document}
\maketitle

\begin{abstract}
We show that feedforward neural networks with ReLU activation generalize on low complexity data, suitably defined. Given i.i.d.~data generated from a simple programming language, the minimum description length (MDL) feedforward neural network which interpolates the data generalizes with high probability. We define this simple programming language, along with a notion of description length of such networks. We provide several examples on basic computational tasks, such as checking primality of a natural number. For primality testing, our theorem shows the following and more. Suppose that we draw an i.i.d.~sample of $n$ numbers uniformly at random from $1$ to $N$. For each number $x_i$, let $y_i = 1$ if $x_i$ is a prime and $0$ if it is not. Then, the interpolating MDL network accurately answers, with \edits{probability} $1- O((\ln N)/n)$, whether a newly drawn number between $1$ and $N$ is a prime or not. Note that the network is not \textit{designed} to detect primes; minimum description learning \textit{discovers} a network which does so. Extensions to noisy data are also discussed, suggesting that MDL neural network interpolators can demonstrate tempered overfitting.
\end{abstract}

\section{Introduction}

Understanding why neural networks generalize well on unseen data is an enduring mystery in the field. For many datasets seen in practice, massively overparametrized neural networks are fit to near-zero training error, yet still generalize on test examples. At the same time, many neural network architectures are capable of fitting pure noise \cite{zhang2021understanding}, yet clearly cannot generalize on these datasets. Classical complexity descriptions from statistical learning theory such as VC dimension \cite{bartlett2003vapnik, sontag1998vc} cannot explain this phenomenon, as VC dimension is distribution-independent. Given this, it is natural to make structural assumptions about the data. For example, in many real-world datasets for which deep learning is deployed (e.g., computer vision or natural language processing), the data has apparent structure with very low levels of noise. In this paper, we prove generalization guarantees on data of low complexity, with zero noise. We introduce a simple programming language and a notion of description length for neural networks. Using these notions, we show that for data generated from a short program in this language, the MDL feedforward neural network interpolating the data has low test error rate with high probability.

\subsection{Main Results}

To capture the notion of low complexity data, we define \textit{simple neural programs} (SNPs). SNPs are simple programs which can define variables and manipulate them with basic operations. They consist of a sequence of statements, and intuitively they may be thought of restricted Python programs. Control statements such as {\tt for} loops and {\tt if} statements are also allowed. For example, checking whether a number is prime can be solved by an SNP. The following snippet gives pseudocode for checking whether an input $n$ is prime or not. Section \ref{sec:defining_snps} gives a full definition of SNPs, with many more examples. For now, the syntax of the language can be interpreted as in Python.

\begin{minted}{python}
input n
for i = 2,...,n:
    for j = 2,...,n:
        prod = i*j
        prod_equals = (prod == n)
        res = res + prod_equals
output = (res > 0)
return output
\end{minted}
Our analysis begins with the observation that every SNP $\tP$ can be encoded as a feedforward neural network $F_{\tP}$ with ReLU nonlinearity.

\begin{theorem}[Thm.~\ref{thm:sneuralprograms_to_neuralnetworks}, Simplified]
Let $\tP$ be a SNP comprised of statements $(\tS_1,\dots,\tS_L)$. Let $\tP$ take in inputs $(x_1,\dots,x_I) \in [N]^{I}$, where $[N] = \set{1,\dots,N}.$  Then for each $N,$ there is a feedforward neural network $F_{\tP,N}$, which agrees with the program for all inputs in $[N]^{I}.$ 
\end{theorem}

Several previous works show that certain neural network architectures, particularly transformers, can model basic programs \cite{weiss2021thinking,lindner2024tracr,perez2021attention, giannou2023looped}. Furthermore, folklore says that various neural network architectures are equivalent to boolean circuits (see e.g. \cite{siegelmann1992computational, shalev2014understanding, livni2014computational} and the universal approximation literature \cite{hornik1989multilayer, cybenko1989approximation, barron1993universal}). For example, Theorem 2 of \cite{livni2014computational} states that any function computable by a Turing machine in $T(d)$ operations can be expressed by a threshold unit neural network of depth $O(T(d))$ and size $O(T(d)^2)$. The proof exploits circuit complexity bounds and the fact that logical gates can be expressed by threshold neural networks.

In light of these results, it is not so surprising that Theorem \ref{thm:sneuralprograms_to_neuralnetworks} holds. An advantage of our theorem is that it provides an explicit conversion between a \textit{simple programming language} and deep feedforward neural networks. Simple programming languages are more interpretable than boolean circuits and allow users to easily express interesting examples like the aforementioned prime numbers example. Most related to Theorem \ref{thm:sneuralprograms_to_neuralnetworks} are the results of \cite{weiss2021thinking}, which defines a explicit programming language that can be compiled into transformer architectures. However, the language is fairly restrictive as Section 3.1 of \cite{zhou2023algorithms} discusses. Furthermore, our constructed neural networks are efficiently describable. Under a simple compression scheme, we show that for a SNP $(\tS_1,\dots,\tS_L)$ of length $L$, the parameters of $F_{\tP,N}$ can be compressed into a sequence of bits polynomial in the length and other simple attributes of the program. The compression scheme allows repetitions of a substring of bits to be replaced by the substring and the number of repetitions. This motivates a notion of description length of a neural network which is roughly given by the minimum compression length of its parameters, leading to the following result.

\begin{proposition}[Prop.~\ref{prop:efficient_conversion}, simplified]
\label{prop:efficient_conversion_simp}
Let $\tP$ be a SNP of length $L$, with $V$ variables, which outputs a result $\tP(x)$ for each input $x \in [N]^I$. Suppose for any input in $[N]^I$, the maximum runtime value of a variable is at most $B(N).$ Then $F_{\tP,N}$ has description length at most $O(L^3V^2\ln B(N))$. 
\end{proposition}

Putting these two results together, we obtain the main result of this paper.

\begin{theorem}[Thm.~\ref{thm:generalization_bound}, simplified]
\label{thm:generalization_simp}
Consider a SNP $\tP$ satisfying the assumptions of Prop. \ref{prop:efficient_conversion_simp}. Fix $\e > 0,\delta \in (0,1)$ and let
\[
n = \Theta\left(\frac{L^3V^2 \ln B(N) + \ln 1/\delta}{\e} \right).
\]
Suppose we observe i.i.d. data $(x_i,y_i),i=1,\dots,n$ 
with $x_i$ drawn from some distribution on $[N]^I$, and $y_i = \tP(x_i).$ Let $\hat{f}_{\MDL}$ be the MDL neural network interpolating the data. Then for $N$ large enough, with probability at least $1-\delta$, the error rate of $\widehat{f}_{\MDL}$ on a uniformly chosen test point is at most of $\e.$
\end{theorem}

If there are multiple minimum description length interpolators, the theorem applies to all of them. For easier interpretation, the idea behind Theorem \ref{thm:generalization_simp} yields an simpler averaged generalization guarantee.

\begin{corollary}[Cor. \ref{cor:avg_gen}, simplified]
Consider a SNP $\tP$ and a dataset $\set{(X_i,Y_i)}_{i=1}^n$ as in Theorem~\ref{thm:generalization_simp}, where $n$ is now generic. Let $\widehat{f}_{\MDL}$ be an interpolating minimum-description length neural network and $x$  a new sample from $\mu$. Then
\[
\Prob\left( \widehat{f}_{\MDL}(x) \neq \tP(x) \right) = O\left(\frac{L^3V^2 \ln B(N)}{n} \right).
\]
\end{corollary}

To demonstrate these results, consider the prime-checking program. Suppose we randomly choose $n$ many integers from $[N]$ and output whether the integer is prime or not. Then $\widehat{f}_{\MDL}$, with high probability, has error rate $ O\left(\frac{\ln N}{n}\right)$. Recall that the density of the primes among the first $N$ natural numbers is $(\ln N)^{-1}$ by the prime number theorem. Therefore, with $n \gg (\ln N)^2$, $\widehat{f}_{\MDL}$ on a typical dataset classifies both primes and non-primes correctly with high accuracy. For the details of this and other examples, see Section \ref{examples}.

To prove Theorem \ref{thm:generalization_simp}, we show that the number of neural networks of description length at most $s$ is at most exponentially large in $s$. A simple probabilistic argument then shows generalization of the minimum description length interpolator. Section \ref{sec:generalization} gives the proof of this theorem with applications to several examples. The proof strategy may be extended to other definitions of description length, letting us derive results similar to Theorem \ref{thm:generalization_simp} by considering variations of simple neural programs and the description length measure. In particular, different setups may be more natural for different neural network architectures beyond feedforward neural networks. Section \ref{sec:noisy_data} describes extensions to interpolation on noisy datasets.

\subsection{Related Work}
\label{sec:litreview}

\paragraph{Structured Data \& Neural Networks} Several works exploit structural assumptions on the data to provide generalization guarantees. In the setting of binary classification \cite{brutzkus2017sgd, li2018learning}, it is shown that the empirical risk minimizer of a two layer neural network trained with stochastic gradient descent generalizes; \cite{brutzkus2017sgd} assumes the data $\set{(x_i,y_i)}_{i=1}^n$ is linearly separable while \cite{li2018learning} assumes the supports of the features $x$ are disjoint. In a different direction, \cite{goldt2020modeling} analyze learning of two layer neural networks where the data is generated from a low dimensional manifold, with labels depending only on the position within the manifold. \cite{chen2022nonparametric} study deep ReLU networks for nonparametric regression tasks under similar setting, inspired by the manifold hypothesis, while \cite{malach2018provably, abbe2021staircase} utilize hierarchical assumptions on the data. \cite{abbe2021staircase} show that so called ``staircase functions" can be learned efficiently using stochastic coordinate descent, while \cite{malach2018provably} consider image-valued data generated by iteratively refining a coarse image and provide new algorithms for learning deep convolutional neural networks. See \cite{mezard2023spin} for further references and a connection with the spin glass literature. See also \cite{arora2018stronger} for a related compression approach to generalization.

In addition to statistical learning problems, neural networks can be used for data compression purposes. The goal is to use certain architectures, particularly those from generative models, to compress image, video, text or otherwise structured data with high fidelity.
\cite{yang2023introduction} provides an overview of this subfield; see the references therein for further details.

\paragraph{Low Complexity Assumptions in Learning} Theorem \ref{thm:generalization_simp} can be seen as a generalization guarantee for minimum description learning with neural network architectures and low-complexity data. Minimum Description Learning (MDL) \cite{rissanen1983universal, grunwald2007minimum, barron1998minimum} is a paradigm for inductive learning with relations to classical topics in computer science and learning theory, especially algorithmic probability and Solomonoff induction \cite{solomonoff1964formal, li2008introduction}. For prediction tasks, it suggests that the predictor which can be described in the least number of bits should be used. Several recent works have re-considered minimum description learning and related ``low-complexity" patterns in light of modern machine learning. The paper \cite{manoj2023interpolation} studies minimum description length rules for a universal description or programming language and its generalization properties when the rule is forced to interpolate the training data, which is similar to our setting. The paper shows MDL learning rules display \textit{tempered overfitting}, where the generalization error is suboptimal, but better than random guessing. The work does not specialize to neural networks, however. \cite{abbe2023generalization} show that certain neural networks trained to learn Boolean functions on strongly out-of-distribution data learn ``minimum degree interpolators''. Relatedly, \cite{grau2024learning} investigate the applications of Solomonoff induction for training neural networks in meta-learning tasks. A few papers combine MDL-type ideas and neural networks. \cite{schmidhuber1997discovering} discusses methods for learning neural networks with low Kolmogorov complexity and high generalization capability, based on universal priors. \cite{hochreiter1997flat} hypothesize that neural networks which are \textit{flat minima} of the loss landscape generalize well, using an argument based on MDL. \cite{hinton93keeping} propose practical methods to implement the MDL principle when training feedforward neural networks. In a similar direction to us, \cite{lan2022minimum} provide empirical results about minimum description length neural networks, for formal language data.

Several recent works have also demonstrated the role of low complexity in trained neural networks. \cite{valle2018deep, mingard2019neural, teney2024neural, bhattamishra2022simplicity, razin2024understanding} all show that certain randomly initialized neural networks are biased towards representing ``low complexity'' functions. For example, \cite{mingard2019neural} show that one layer perceptions are biased towards low entropy functions, while \cite{teney2024neural, bhattamishra2022simplicity} consider transformer architectures. \cite{mingard2023deep, goldblum2023no} also show similar empirical results for neural networks trained with gradient descent. Similarly to our work, \cite{mingard2023deep} considers the effect of low complexity data in their analysis.

\paragraph{Statistical Guarantees for Neural Networks on Noisy Data} As discussed in the introduction, statistical guarantees for neural network architectures, particularly on noisy data, are the subject of intense research. These questions are motivated by empirical observations like interpolation and double descent \cite{zhang2021understanding, belkin2019reconciling, nakkiran2019deep}. More recently, a growing literature seeks to characterize to what extent neural networks should generalize on noisy data, by defining three regimes: \textit{benign, tempered} and \textit{catastrophic overfitting} \cite{mallinar2022benign}. Theoretical results along these lines, although insightful, are generally restricted to special cases such as linear \& ridge regression, the kernel regime of neural networks, or two layer neural networks \cite{bartlett2020benign,tsigler2023benign, kornowski2023tempered, bartlett2021deep}. Our Theorem \ref{thm:generalization with noise} shows that minimum description length neural network interpolators display tempered overfitting on corrupted low complexity data. In particular, we show that the generalization error for the minimum description length interpolator on a dataset of size $n$, with $\rho$ fraction of the labels corrupted arbitrarily, behaves like $O(\rho) + O(1/n)$.

While completing this work, we were made aware of the recent paper \cite{harel2024provable} which proves neural network interpolators with minimum number of weights exhibit tempered overfitting. The results are of a similar flavor to Theorem \ref{thm:generalization with noise}. \cite{harel2024provable} considers a setting where the data is generated from a teacher model, a fixed neural network with noisy binary outputs. The interpolators considered are \textit{binary threshold networks}, which have binary parameters, binary inputs, and additional thresholding weights. \cite{harel2024provable} shows that the worst-case average generalization error of the minimum size network interpolator behaves approximately like $\rho \ln 1/\rho + o_n(1)$. With independent noise, the generalization error behaves like $\rho + o_n(1)$. See Theorem~\ref{thm:generalization with noise} and Remark~\ref{remark:harel_comparison} for further details about this work and comparison to ours.

\paragraph{Transformers as Algorithm Approximators} 

Transformers \cite{vaswani2017attention} are a neural network architecture behind much of the success of large language models. Language models based on transformers and similar architectures demonstrate a remarkable generalization ability called \textit{in-context learning}: the model can perform new tasks when given access to a small number of training and test examples \cite{brown2020language, garg2022can}. Similarly to our connection between feedforward neural networks encoding simple programs, \cite{bai2023transformers,lin2023transformers,mei2023deep, giannou2023looped} show that transformers can approximate certain types of algorithms in-context, including statistical algorithms such as least squares. Relatedly \cite{zhou2023algorithms} consider transformer performance for length-generalization tasks, such as training the transformer on 3 digit addition problems and testing it on 10 digit addition. Based on extensive empirical results, they conjecture that transformers tend to length- generalize on tasks that can be solved by a short programming language called RASP \cite{weiss2021thinking} which emulates a computational model of transformer architectures. \cite{lindner2024tracr} study the RASP language further, and show how simple RASP programs may be converted back into transformers. This is similar in spirit to Theorem \ref{thm:generalization_simp}, although our results do not apply to length-generalization.

\paragraph{Turing Completeness of Neural Networks and Related Results} 
Foundational results in the field of neural networks going back to \cite{mcculloch1943logical} demonstrate that NNs can not only universally approximate functions, but they can also emulate universal models of computation. \cite{siegelmann1992computational} showed that single-layer rational-weight recurrent neural networks (RNNs) can compute any computable function; similarly \cite{balcazar1997computational} shows the equivalence between some RNNs and Turing machines, expressing the computational power of RNNs using complexity of weights in terms of Kolmogorov complexity. Many recent papers improve on these results, and also demonstrate the ability of modern neural network architectures to represent Turing machines, automata, and similar computational models \cite{perez2019turing, perez2021attention, wei2022statistically, liu2022transformers, svete2024transformers}. \cite{wei2022statistically} show transformers can approximate Turing machines of bounded computation time, and establish bounds on the sample complexity of the problem. \cite{liu2022transformers} show similar approximation results for finite state automata. \cite{chen2017recurrent, mali2023computational} consider questions  on the computational complexity of using recurrent neural networks to represent computational models and formal languages, while \cite{stogin2024provably} consider other architectures to approximate push-down automata. \cite{clark2020transformers} also details a connection with logic. See \cite{sanford2023representational, strobl2024formal} for additional references.

\section{Defining a Programming Language}
\label{sec:defining_snps}

A \textit{simple neural program} (SNP) $\tP$ consists of a \textit{variable context}, specifying all the variables in the program, along with any sequence of statements to be described. Examples of the syntax are described below each of the statements. A \textit{variable context} for $\tP$ describes the set of variables to be manipulated in $\tP$. All variables in the program must be declared in the variable context. It is comprised of a sequence of statements of two types: 
\begin{itemize}
    \item {\bf Input statements.} These statements define a variable which is taken as input into the program, and do not have a defined value at the beginning of the SNP. The syntax is {\tt input <variable name>}. All variable names are distinct.
    \item {\bf Variable initialization statements.} All variables need to be either nonnegative integer valued or boolean valued (i.e., encoded by zero or one). {\it In particular, throughout the runtime of the program, all variables are enforced to be nonnegative integer valued.} Variables must be initialized with a fixed value. The syntaxes for the two types are {\tt int <variable name> = <value>} and {\tt bool <variable name> = <value>}.
\end{itemize}  
Here is an example. 
\begin{minted}[linenos]{python}
input x
int a = 5
bool b = 1
...
\end{minted}
Following the variable context is any sequence of the statements described below. The statements may only reference variables defined in the variable context of $\tP$. When referring to SNP commands and constructions, we will often write with the \texttt{monospace} font. Unless otherwise specified, all constants referred to below are integers.

\begin{enumerate}    
    
    \item \textit{Value assignment.} A given variable may be assigned a fixed nonnegative integer or the value of another variable in the  program. The syntax is {\tt <variable name> = <value or variable name>}.
    \begin{minted}{python}
    int x = 0
    int a = 0
    x = 1
    x = a
    ...
    \end{minted}
    \item {\it For loops.} {\tt For} loops increment an existing counter variable by $1$ in each repetition; the range of the loop may have a variable start and variable end. The syntax is: {\tt for <counter variable> = <initial value or variable>,...,<final variable or value>:}.
    Following a {\tt for} loop is a \textit{clause} $\tC$, i.e., a block of SNP statements. In the example below, lines $6$ and $7$ comprise $\tC$. Note that $\tC$ may be seen as an SNP with the same variable context as $\tP$. Clauses may not modify the counter variable or the final variable inside the clause. 
    \begin{minted}[linenos]{python}
    int s = 1
    int n = 10
    int i = 1
    int res = 0
    for i = s,...,n:
        res = 0
        res = res + 1
    ...
    \end{minted}
    Clauses may contain further {\tt for} loops. If the program $\tP$ nests $d$ {\tt for} loops, we say that $\tP$ has \textit{depth} $d$. For example, the following snippet contains a double {\tt for} loop. {\tt for} loops which are not contained in another {\tt for} loop are said to be top-level. Otherwise, the loop is nested. For example, the loop on line $5$ is top-level, while the one on line $6$ is nested. More generally, any SNP statement $\tS_i$ which is not contained in a {\tt for} loop is said to be \textit{top-level}.
    \begin{minted}[linenos]{python}
    int n = 10
    int i = 1
    int j = 1
    int res = 0
    for i = 1,...,n:
        for j = 1,...,n:
           res = i + j 
    ...
    \end{minted}    
    \item \textit{If statements.} {\tt if} statements must be of the following form: if a boolean variable is equal to $1$, update a variable $c$ with a quantity $a$; else, with another quantity $b$. The quantities may be variable or constant. The syntax should be clear from the example below. We do not allow for more complicated {\tt if} statements which have multiple lines within the if clause or else clause. 
    \begin{minted}[linenos]{python}
    int a = 2
    int b = 5
    int c = 3
    bool cond = 1
    c = a if cond else b
    ...
    \end{minted}
    
    \item \textit{Return statement.} This returns an existing variable in the program. The program ends with a return statement. The syntax is simply {\tt return <variable name>}.
    \begin{minted}{python}
    input x
    ...
    return x
    \end{minted}
    
    \item \textit{Basic operations.} We allow only two basic operations: Addition of a variable with a fixed integer, and multiplication of a variable by a fixed nonnegative integer. The output of every operation must be assigned to an existing variable in the program. The syntax is {\tt <output variable name> = <variable name> + <constant>} for addition, and \newline {\tt <output variable name> = <constant> * <variable name>} for multiplication.
    \begin{minted}{python}
    int a = 2
    int b = 3
    int c = 0
    c = a + b
    c = a + 2
    a = 2 + 3
    a = 2 * b
    ...
    \end{minted}
    
    \item \textit{Unary operators.} We allow the following unary operators: checking equality to a nonnegative constant, checking greater than a nonnegative constant, and checking less than a nonnegative constant. The syntax in each of these cases is given by {\tt <bool variable name> = (<int variable name> == <constant>)}, and similarly for {\tt <, $<=$, >, $>=$}.
    \begin{minted}{python}
    int a = 1
    int c = 3
    bool b = 0
    b = (a == 0)
    b = (c > 3)
    b = (c < 4)
    ...
    
    \end{minted}
    
    \item \textit{Binary operators.} We allow the addition and subtraction of two numbers (as long as the output is nonnegative), along with comparisons of two variables with $=,<,>, <=, >=$. The syntax is {\tt <output variable name> = <variable1 name> + <variable2 name>}, {\tt <output variable name> = (<variable1 name> == <variable2 name>)}, and similarly for the other operations.
    \begin{minted}{python}
    input x
    int a = 1
    int c = 3
    bool b = 0
    int d = 0
    b = (a == c)
    d = x + c
    ...
    \end{minted}
\end{enumerate}
We define the \textit{length} of a simple neural program to be the number of statements in the program (not including input or variable initialization statements). This also counts every statement in the clause of every {\tt for} loop in the program. When the variables and constants of the SNP do not exceed a constant $B$ throughout the runtime of the program, we say the program is $B$-bounded.

\paragraph{Composing programs} Let $\mathbb{N}_0$ be the set of natural numbers including zero. A SNP with an input $\mathbf{x} \in \mathbb{N}_0^I$ can be thought as a function with domain $\mathbb{N}_0^I$, where $I$ is the dimensionality of the input to $\tP$. Thus, it is possible to compose SNPs together: given a program $\tP_1(i_1,\dots,i_k)$ with inputs $i_1,\dots,i_k$, we can define another program $\tP_2$ with variable context $\tV_2$, one of whose statements is given by  
\[
\tt{y} = \tP_1(\tt{x}_1,\dots, \tt{x}_k), 
\]
for $\tt{x}_1,\dots, \tt{x}_k \in \tV_2, y \in \tV_2$. Call a program $\tP$ \textit{composite} if any of its statements is a call to another SNP. {\it We disallow recursive calls in composite programs.} Specifically, a program $\tP_2$ is allowed to call a program $\tP_1$ only if $\tP_1$ has been defined \textit{prior} to $\tP_2$. If all the statements of a program $\tP$ are primitive (one of the 7 types referenced above), it is called \textit{atomic}. All results in this paper will apply to atomic programs. 

Occasionally, it is simpler to write a program $\tP$ as a composite program. Example \ref{ex:primenumber_program} below shows such a case. However, it is easy to reduce a composite program to an atomic one, by expanding out the lines of the subprograms, and combining variable contexts. Consider a composite program $\tP_2$ with variable context $\tV_2$, and an atomic program $\tP_1(i_1,\dots,i_k)$ with variable context $\tV_1$, where $i_1,\dots,i_k$ are input variables contained in $\tV_1$. Suppose that one of the lines of $\tP_2$ is
\[
\tt{y} = \tP_1(\tt{x}_1,\dots, \tt{x}_k)
\]
for $\tt{x}_1,\dots, \tt{x}_k \in \tV_2, y \in \tV_2$. We may consider an atomic program $\tP_{\atomic}$ which is equal to $\tP_2$ as a function, defined as follows.
\begin{itemize}
    \item The variable context of $\tP_{\atomic}$ is $\tV_2 \cup \tV_1\backslash \set{i_1,\dots,i_k}$. It has the same inputs as $\tV_2$.
    \item Replace the line  $\tt{y} = \tP_1(\tt{x}_1,\dots, \tt{x}_k)$ by the following sequence of lines:
    \begin{enumerate}
        \item The variable initialization statements of $\tP_1$ (which are included at the start of $\tP_{\atomic})$.
        \item The non-\texttt{return} statements of $\tP_1$, substituting all input variables $i_1,\dots,i_k$ by $\tt{x}_1,\dots, \tt{x}_k$,
        \item {\tt y = result}, where \texttt{result} denotes the return variable of $\tP_1$.
    \end{enumerate} 
\end{itemize}
See Ex.~\ref{ex:primenumber_program} for an example of this reduction. Henceforth, only atomic SNPs are considered.

\begin{example}[Integer multiplication]
Multiplying two integer variables is not a primitive in the programming language, but it can be easily implemented using a {\tt for} loop. We can think of this as a function $\texttt{multiply}(x,y)$ which takes in two inputs $x,y \in \bb{N}.$

\begin{minted}{python}
input x
input y
int i = 0
int res = 0
for i = 1,...,x:
    res = res + y
return res
\end{minted}

\end{example}

\begin{example}[Primality testing]
\label{ex:primenumber_program}
Let $N$ be fixed. For any $n \leq N$, checking whether $n$ is a prime number can be expressed as an SNP. 

\begin{minted}[linenos]{python}
input n
int i = 2
int j = 2
int res = 0
int prod = 0
int t = 0
bool output = 0
bool prod_equals = 0
for i = 2,...,n:
    for j = 2,...,n:
        prod = multiply(i,j)
        prod_equals = (prod == n)
        res = res + prod_equals
output = (res > 0)
return output
\end{minted}
This is a composite SNP, since we call the non-primitive function \texttt{multiply} on line $11$. See Section \ref{sec:full_programs} for the program written out fully as an atomic program. 


\end{example}

\begin{example}[Fibonacci numbers]
\label{ex:fib}
Outputting the $n$th Fibonacci number is also simple. 

\begin{minted}[linenos]{python}
input n
int x = 0
int y = 1
int temp = 0
int i = 1
int loop_var = 0
loop_var = n-1
for i = 1,...,loop_var:
    temp = y
    y = y + x
    x = temp
return y
\end{minted}
\end{example}
The program has a variable context of size $6$, with length $6.$

\subsection{A nested representation of simple neural programs}

A SNP $\tP$ with a variable context $\tV$ can be written as a sequence of statements $(\tS_1,\dots,\tS_L)$. Another way to describe the program is to enumerate all the top-level statements, those that are not contained within a {\tt for} loop.

\begin{definition}[Top-level representation of $\tP$]
Given an SNP $\tP = (\tS_1,\dots,\tS_L)$ with variable context $\tV$, enumerate all top level {\tt for} loops and their clauses by $(\tS_{n_i},\tC_i)$ for some subsequence $\set{n_i}$ indicating the locations of top-level {\tt for} loops. 
The top-level representation of $\tP$ is the unique sequence $(\tO_1,\dots,\tO_k)$ where each $\tO_i$ is either a top-level statement or a {\tt for} loop clause pair $(\tS_{n_j},\tC_j)$ and for $i < j, \tO_i$ appears before $\tO_j$ in the program. 
\end{definition}

Consider Example \ref{ex:fib}. The program can be written as the sequence $(\tO_1, \tO_{2}, \tO_3)$ where $\tO_1$ is the statement on line $7,$ $\tO_3$ is the statement on line $12$, and $\tO_2$ is the tuple $(\tS_2, \tC_2)$ where $\tS_2$ is the {\tt for} loop on line $8$ and $\tC_2$ is its clause comprising lines 9-11. In Example \ref{ex:primenumber_program}, the top-level representation is $((\tS_{9},\tC),\tS_{14},\tS_{15})$ where $\tC$ is the clause comprising statements 10-13.

\section{Encoding SNPs by Feedforward Neural Networks}
\label{sec:snps_to_neuralnetworks}

The fundamental result for our simple neural programming language is that any atomic program can be converted into a fully-connected feedforward neural network with ReLU nonlinearity. This is perhaps unsurprising given the literature outlined in Section \ref{sec:litreview} on the Turing completeness of neural networks, but the encoding of the program by a neural network is efficiently describable in a way we outline in a Section \ref{sec:description_complexity}. We consider feedforward neural network architectures $F_{\bm{\theta}}$ which are compositions of affine functions and the ReLU nonlinearity $\sigma(x) = \max(x,0)$,
\begin{align*}
F_{\bm{\theta}}(\bx) & = g_D \circ \sigma \circ g_{D-1} \circ\sigma \circ\cdots \circ g_2 \circ \sigma \circ g_1(\bx) \\
g_i(\bx) & = W_i\bx + b_i,
\end{align*}
which may be parametrized by its sequence of layer weights and biases $\bm{\theta} = (\theta_1,\dots,\theta_D),$ where $\theta_k = (W_k,b_k)$. 

We will construct an encoding of SNPs as such networks. Every variable in the program is stored as a unique node in the neural network, and every statement of the simple neural program corresponds to a sequence of consecutive layers in the network. The ordering of the layers of the neural network reflects the ordering of the statements in the simple neural program. The construction will be inductive on the depth of the program; recall that the \textit{depth} of a program is maximum number of times that {\tt for} loops are nested within each other. Consider first the case of depth zero programs.

\subsection{Base case: depth zero SNP conversion}
\label{sec:base_case_conversion}


Consider a depth-zero SNP $\tP = (\tS_1,\dots,\tS_L)$ with a variable context \texttt{V} of size $V,$ indexed by $\mathbf{x} = (x_1,\dots,x_V).$ Let $\tP$ take inputs in $[N]^I.$ Throughout this section, assume that the maximum possible value of a variable during the program is bounded by $B := B(N).$ Note that $B$ \textit{does not} bound the time complexity of the algorithm. It just controls the values of the variables throughout the program. Each statement $\tS_i$ in the program will be encoded as a composition of layers $f_{\tS_i} = g_{i,k_i} \circ g_{i,k_i-1}\circ \cdots \circ g_{i,2} \circ g_{i,1}$, where $g_{i,l}(\mathbf{y}) = \sigma(W^{(i,l)} \mathbf{y}  + b^{(i,l)})$, and the non-linearity $\sigma$ acts component-wise. Each sequence of layers $f_{\tS_i}$ is a map from $\bb{R}^V$ into  $\bb{R}^V$. We will occasionally write $f_{\tS_i,B}$ to emphasize the dependence of the parameters on $B$. The $f_{\tS_i}$ are strung together to act on $\mathbf{x}$, so that the program $\tP$ corresponds to the neural network
\[
f_{\tS_L} \circ f_{\tS_{L-1}} \circ \dots \circ f_{\tS_1}(\mathbf{x}).
\]
The individual layers $g_{i,l}$ which define $f_{\tS_i}$ may change dimension, depending on the statement $\tS_i.$ The next section will explicitly define the individual layers, with the goal of showing that the sequence of layers $f_{\tS_i}$ agrees with the statement $\tS_i$ as functions $\bb{N}_0^{V} \rightarrow \bb{N}_0^{V}.$

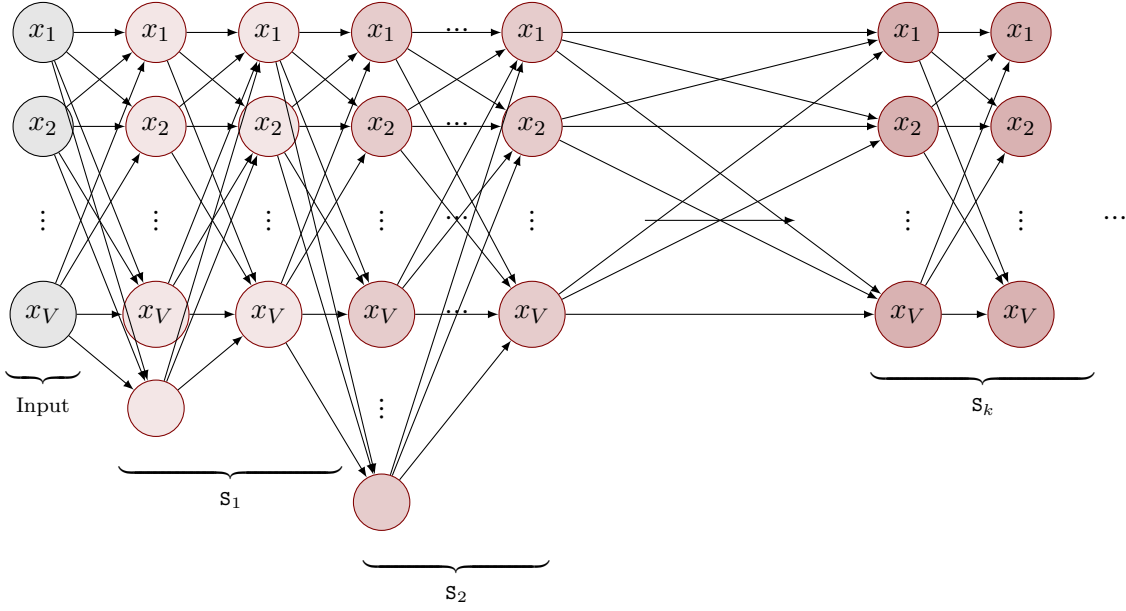
\begin{figure}[t]
\centering

\begin{tikzpicture}[node distance=1.25cm, auto, align = center]

\tikzstyle{input} = [circle, draw=black, fill=gray!20, minimum size = 0.75cm]
\tikzstyle{hidden} = [circle, draw=maroon, fill=maroon!10, minimum size = 0.75cm]
\tikzstyle{hidden2} = [circle, draw=maroon, fill=maroon!20, minimum size = 0.75cm]
\tikzstyle{output} = [circle, draw=maroon, fill=maroon!30, minimum size = 0.75cm]

\node[input] (x1) {$x_1$};
\node[input, below of=x1] (x2) {$x_2$};
\node[below of=x2] (vdots) {$\vdots$};
\node[input, below of=vdots] (x3) {$x_V$};

\node[hidden, right of=x1, node distance=1.5cm] (h1) {$x_1$};
\node[hidden, below of=h1] (h2) {$x_2$};
\node[below of=h2] (vdots2) {$\vdots$};
\node[hidden, below of=vdots2] (h3) {$x_V$};
\node[hidden, below of=h3] (h4) {};

\node[hidden, right of=h1, node distance=1.5cm] (s21) {$x_1$};
\node[hidden, below of=s21] (s22) {$x_2$};
\node[below of=s22] (vdots3) {$\vdots$};
\node[hidden, below of=vdots3] (s23) {$x_V$};

\node[hidden2, right of=s21, node distance=1.5cm] (s31) {$x_1$};
\node[hidden2, below of=s31] (s32) {$x_2$};
\node[below of=s32] (vdots4) {$\vdots$};
\node[hidden2, below of=vdots4] (s33) {$x_V$};
\node[below of=s33] (vdots41) {$\vdots$};
\node[hidden2, below of=vdots41] (s34) {};

\node[right of=vdots4, xshift = -0.25cm] (s3s4dots) {$\hdots$};
\node[right of=s31, xshift = -0.25cm] () {$\hdots$};
\node[right of=s32, xshift = -0.25cm] () {$\hdots$};
\node[right of=s33, xshift = -0.25cm] () {$\hdots$};

\node[hidden2, right of=s31, node distance=2cm] (s41) {$x_1$};
\node[hidden2, below of=s41] (s42) {$x_2$};
\node[below of=s42] (vdots5) {$\vdots$};
\node[hidden2, below of=vdots5] (s43) {$x_V$};

\node[right of= vdots5, node distance=0.5cm] (arrow) {};
\draw[->] (arrow) ++(1,0) -- ++(2,0);

\node[output, right of=s41, node distance=5cm] (o1) {$x_1$};
\node[output, below of=o1] (o2) {$x_2$};
\node[below of=o2] (vdots4) {$\vdots$};
\node[output, below of=vdots4] (o3) {$x_V$};

\node[output, right of=o1, node distance=1.5cm] (o21) {$x_1$};
\node[output, below of=o21] (o22) {$x_2$};
\node[below of=o22] (vdots5) {$\vdots$};
\node[output, below of=vdots5] (o23) {$x_V$};

\node[right of=vdots5] (lastnode) {$\hdots$};

\foreach \i in {1,2,3}
  \foreach \j in {1,2,3,4}
    \draw[->, opacity = 0.2] (x\i) -- (h\j);

\foreach \i in {1,2,3,4}
  \foreach \j in {1,2,3}
    \draw[->, opacity = 0.2] (h\i) -- (s2\j);

\foreach \i in {1,2,3}
  \foreach \j in {1,2,3,4}
    \draw[->, opacity = 0.2] (s2\i) -- (s3\j);

\foreach \i in {1,2,3,4}
  \foreach \j in {1,2,3}
    \draw[->, opacity = 0.2] (s3\i) -- (s4\j);

\foreach \i in {1,2,3}
  \foreach \j in {1,2,3}
    \draw[->, opacity = 0.2] (s4\i) -- (o\j);

\foreach \i in {1,2,3}
  \foreach \j in {1,2,3}
    \draw[->, opacity = 0.2] (o\i) -- (o2\j);

\node[below of=x3, node distance=1cm] {$\underbrace{\hspace{1cm}}_{\text{Input}}$};
\node[below of=h4, node distance=1cm, xshift = +1cm] (brace1) {$\underbrace{\hspace{3.0cm}}_{\tS_1}$};
\node[below of=s34, node distance=1cm, xshift = +1.cm] (brace1) {$\underbrace{\hspace{2.5cm}}_{\tS_2}$};
\node[below of=o3, node distance=1cm, xshift = +1cm] (brace3) {$\underbrace{\hspace{3.0cm}}_{\tS_{k}}$};

\end{tikzpicture}

\caption{Neural network encoding a depth zero program $\tP$. The input vector $\mathbf{x}$ is described by the variable context of the program. Sequences of layers correspond to statements $\tS_i$ in $\tP$, which may have different widths. Additional nodes in the layers do not correspond to variables in $\tV$, and instead store the outputs of intermediate computations.}
\label{fig:neural_network}
\end{figure}

\paragraph{Statement encodings}




The variable context $\texttt{V}$ of $\tP$ defines the input vector $\mathbf{x}$ of the neural network. All variable declaration statements such as {\tt <var type> var = c} initialize the component of $\mathbf{x}$ corresponding to the variable $\tt{var}$ with the value $c$. \texttt{Input} statements define which components are free variables. 

\begin{itemize}

\item \textit{Value assignment.} To set the $i$th variable equal to a fixed constant $c \geq 0$, we use $W = I - e_ie_i^\top$ and $b = ce_i$, where $e_i$ denotes the column vector whose $i$th component is $1$ and the rest are $0$, and $e_i^\top$ is the transpose of $e_i$. For setting the $i$th variable equal to the $j$th variable, we use $W = I - e_ie_i^\top + e_ie_j^\top$ and $b = 0$.


\item \textit{Basic operations.} 
To sum the $j$th variable with a constant $c$ and assign the output to variable $i$, we use $W = I - e_ie_i^\top + e_ie_j^\top$ and $b = ce_i.$ Similarly, multiplying the $j$th variable with a nonnegative constant $c$ and assigning the output to variable $i$ can be encoded with $W = I - e_ie_i^\top + c e_i e_j^\top$ and $b = 0$.


\item \textit{Unary operations.} Consider first the operation which assigns a variable $x_i$ the value $\mathbf{1}\set{x_j = c}$ for another variable $x_j$ and a constant $c \geq 0.$ Encoding this statement and similar statements relies on the identity
\begin{equation}
\label{eq:relu_identity}
\mathbf{1}\set{x = 0} = \relu(x+1) + \relu(x-1) - 2\relu(x)    
\end{equation}
that holds for all $x \in \bb{Z}$. Because $x_j$ is an integer, 
\begin{equation}
\label{eq:equality_check}
\mathbf{1}\set{x_j = c} = \relu(x_j - c + 1) + \relu(x_j - c - 1) - 2\relu(x_j - c).
\end{equation}
This can be expressed in two layers $W^{(1)},b^{(1)}$ followed by $W^{(2)},b^{(2)}$. The first layer creates three temporary variables, which will be indexed at $V+1$, $V+2$, and $V+3$, to store the three numbers $\relu(x_j - c + 1)$, $\relu(x_j - c - 1)$, and $\relu(x_j - c)$. The second layer updates $x_i \leftarrow \relu(\relu(x_j - c + 1) + \relu(x_j - c - 1) - 2\relu(x_j - c))$, which equals $\mathbf{1}\set{x_j = c}$, and deletes the temporary variables. In the following, $W^{(1)}_{r, \cdot}$ denotes the $r$th row of $W^{(1)}$, and so on:
\begin{align*}
W^{(1)},\, b^{(1)} & = 
\begin{cases}
    W^{(1)}_{r,\cdot} = e_r^\top,\, b^{(1)}_r = 0 \quad &\text{for } r=1,\dots,V, \\
    W^{(1)}_{r,\cdot} = e_{j}^\top,\, b^{(1)}_r = -c+1 \quad & r=V+1, \\
    W^{(1)}_{r,\cdot} = e_{j}^\top,\, b^{(1)}_r = -c-1 \quad & r=V+2, \\
    W^{(1)}_{r,\cdot} = e_{j}^\top,\, b^{(1)}_r = -c \quad & r=V+3, \\
\end{cases} \\
W^{(2)} & = 
\begin{cases}
    W^{(2)}_{r,\cdot} = e_{V+1}^\top + e_{V+2}^\top - 2e_{V+3}^\top, \quad &\text{for } r= i, \\
    W^{(2)}_{r,\cdot} = e_{r}^\top, \quad & r = [V]\backslash i, 
\end{cases} \\
b^{(2)} &= 0.
\end{align*}
Next consider assigning $x_i$ the quantity $\mathbf{1}\set{x_j > c}$. This can be encoded via the identity $\relu(x-c) - \relu(x - c-1) = \mathbf{1}\set{x > c}$ that holds for all integers $x$. As before, two layers are required;  the first layer creates two temporary variables to  store $\relu(x_j-c)$ and $\relu(x_j-c-1)$. The second assigns $x_i \leftarrow \relu(\relu(x_j-c) - \relu(x_j - c-1))$ and deletes the temporary variables:
\begin{align*}
W^{(1)},\, b^{(1)} & = 
\begin{cases}
    W^{(1)}_{r,\cdot} = e_r^\top, \, b^{(1)}_r = 0 \quad &\text{for } r=1,\dots,V, \\
    W^{(1)}_{r,\cdot} = e_{j}^\top,\, b^{(1)}_r = -c \quad & r=V+1, \\
    W^{(1)}_{r,\cdot} = e_{j}^\top, \, b^{(1)}_r = -c-1 \quad & r=V+2, 
\end{cases} \\
W^{(2)} & = 
\begin{cases}
    W^{(2)}_{r,\cdot} = e_{V+1}^\top - e_{V+2}^\top, \quad &\text{for } r= i, \\
    W^{(2)}_{r,\cdot} = e_{r}^\top, \quad & r = [V]\backslash i, 
\end{cases}\\
b^{(2)} &= 0.
\end{align*}
The other cases are similar. For example, the case of assigning variable $x_i$ the quantity $\mathbf{1}\set{x_j < c}$ can be encoded using the integer identity
\[
\edits{\relu(c - x_j) - \relu(c - x_j-1)} = \mathbf{1}\set{x_j < c}.
\]

\item \textit{Binary numerical operations.} 
Consider adding/subtracting two variables $x_i,x_j$ and assigning them to variable $x_k$. This is encoded by one layer with parameters.
\[
W = 
\begin{cases}
    W_{r,\cdot} = e_i^\top \pm e_j^\top \quad &\text{for } r= k, \\
    W_{r,\cdot} = e_{r}^\top \quad &\text{otherwise,}
\end{cases}
\]
and $b=0$. Again, we consider only programs such that $x_i - x_j \geq 0$. 

\item {\it Binary logical operations.} Consider checking equality of two variables $x_i,x_j$ and assigning $x_k \leftarrow \mathbf{1}\set{x_i = x_j}.$ By the identity \eqref{eq:relu_identity}, this may be done by taking the difference $x_i - x_j$ and applying  $x \mapsto \relu(x+1)+\relu(x-1) - 2\relu(x)$. As before, this requires two layers. The first layer creates additional variables to store and compute $\relu(x_i - x_j)$, $\relu(x_i - x_j + 1)$, and $\relu(x_i - x_j - 1).$ The second layer calculates $\relu(\relu(x_i - x_j+1)+\relu(x_i - x_j-1) - 2\relu(x_i - x_j))$. Because the argument is a nonnegative integer, the result is the same as $\relu(x_i - x_j+1)+\relu(x_i - x_j-1) - 2\relu(x_i - x_j)$. Explicitly, we use:
\begin{align*}
W^{(1)},\, b^{(1)} & = 
\begin{cases}
    W^{(1)}_{r,\cdot} = e_r^\top, \, b^{(1)}_r = 0 \quad &\text{for } r=1,\dots,V, \\
    W^{(1)}_{r,\cdot} = e_i^\top - e_{j}^\top, \, b^{(1)}_r = 1 \quad &r=V+1, \\
    W^{(1)}_{r,\cdot} = e_i^\top - e_{j}^\top, \, b^{(1)}_r = -1 \quad & r=V+2, \\
    W^{(1)}_{r,\cdot} = e_i^\top - e_{j}^\top, \, b^{(1)}_r = 0 \quad & r=V+3, \\
\end{cases} \\
W^{(2)} & = 
\begin{cases}
    W^{(2)}_{r,\cdot} = e_{V+1}^\top + e_{V+2}^\top - 2e_{V+3}^\top, \quad &\text{for } r= k, \\
    W^{(2)}_{r,\cdot} = e_{r}^\top, \quad & r = [V]\backslash k, 
\end{cases} \\
b^{(2)} &= 0.
\end{align*}
The other cases are similar. For example, checking if a variable $x_i$ is strictly greater than $x_j$ and storing the result in $x_k$ can be done by applying the transformation $\relu(\relu(x) - \relu(x-1))$, using: 

\begin{align*}
W^{(1)},\, b^{(1)} & = 
\begin{cases}
    W^{(1)}_{r,\cdot} = e_r^\top, \, b^{(1)}_r = 0 \quad &\text{for } r=1,\dots,V, \\
    W^{(1)}_{r,\cdot} = e_i^\top - e_{j}^\top, \, b^{(1)}_r = 0 \quad & r=V+1, \\
    W^{(1)}_{r,\cdot} = e_i^\top - e_{j}^\top, \, b^{(1)}_r = -1 \quad & r=V+2, 
\end{cases} \\
W^{(2)} & = 
\begin{cases}
    W^{(2)}_{r,\cdot} = e_{V+1}^\top - e_{V+2}^\top, \quad &\text{for } r= k, \\
    W^{(2)}_{r,\cdot} = e_{r}^\top, \quad & r = [V]\backslash k, 
\end{cases}\\
b^{(2)} &= 0.
\end{align*}
    
\item \textit{If statements}. Suppose the \texttt{if} condition is given by the boolean variable $x_c$. When the \texttt{if} statement is true, suppose the variable $x_i$ is updated by a variable $x_j$; otherwise it is updated by another variable $x_k$. Then we can encode the {\tt if} statement as
\[
x_i \leftarrow x_cx_j + (1-x_c)x_k.
\]
Using the assumption that $x_j,x_k \leq B$, we claim that this transformation is equal to
\[
\relu( (2x_c-1)B + x_j ) + \relu( (2(1-x_c)-1)B + x_k ) - B.
\]
To see this, note that when $x_c = 1$, the expression evaluates to $\relu(B + x_j) + \relu(x_k - B) - B = \relu(B + x_j) - B = x_j$. When $x_c = 0$ the expression evaluates to $\relu(x_j - B) + \relu(x_k + B) - B = \relu(x_k + B) - B = x_k.$ We may also update $x_i$ with constants instead of variables. The formulas below adapt in a straightforward way: 
\begin{align*}
W^{(1)},\, b^{(1)} & = 
\begin{cases}
    W^{(1)}_{r,\cdot} = e_r^\top,\, b^{(1)}_r = 0 \quad &\text{for } r=1,\dots,V, \\
    W^{(1)}_{r,\cdot} = 2Be_c^\top + e_j^\top, \, b^{(1)}_r = -B \quad & r=V+1, \\
    W^{(1)}_{r,\cdot} = -2Be_c^\top \edits{+} e_{k}^\top,\, b^{(1)}_r = B \quad & r=V+2,
\end{cases} \\
W^{(2)}, \, b^{(2)} & = 
\begin{cases}
    W^{(2)}_{r,\cdot} = e_{V+1}^\top + e_{V+2}^\top, \, b^{(2)}_r = -B \quad &\text{for } r= i, \\
    W^{(2)}_{r,\cdot} = e_{r}^\top,\, b^{(2)}_r = 0 \quad & r = [V]\backslash i.
\end{cases}
\end{align*}
    
\item \textit{Return}. This is the last layer in the network, which has one node. Supposing that the desired output is the $i$th variable $x_i,$ take $W = e_i^\top, b = 0.$ 
\end{itemize}

\subsection{Inductive step: For loops}

Now, consider encoding a general depth $d$ SNP $\tP$ with a feedforward network. Suppose $\tP$ has top-level representation $(\tO_1,\dots,\tO_k)$. Assume again that the runtime values of variables are bounded by $B := B(N)$. Each object $\tO_i$ can be mapped to a sequence of neural network layers $f_{\tO_i,B}.$
\begin{itemize}
    \item If $\tO_i$ is an SNP statement, $f_{\tO_i,B}$ is the corresponding layer defined in Section \ref{sec:base_case_conversion}.
    \item If $\tO_i$ is a {\tt for} loop with clause $\tC$, $f_{\tO_i,B}$ is the sequence of layers described in the remainder of this section.
\end{itemize}
Finally, define $F_{\tP,N}$ to be the composition
\[
f_{\tO_k,B} \circ f_{\tO_{k-1},B} \dots \circ f_{\tO_1,B}.
\]

\paragraph{For loop encoding} 
Consider a {\tt for} loop with clause $\tC$, which increments a counter variable $x_i$ from $x_{s}$ to $x_{e}$. The start and endpoints of the loop may also be constants; the layer constructions below adapt straightforwardly. By the inductive hypothesis, $\tC$ is a depth $d-1$ SNP with variable context $\tV$, so there exists a neural network $F_{\tC}$ encoding the program. The prescription involves the following sequence of layers:

    \begin{enumerate}
        \item The first layer ($L_1$) sets the counter variable $x_i$ to the specified start $x_{s}$, and initializes $c \leftarrow 0$, which is not contained inside the variable context $\tV$. Storing the variable $c$ as the $(V + 1)$th variable, ($L_1$) has weight and bias parameters
        \[
        W,\, b = 
        \begin{cases}
            W_{r,\cdot} = e_s^\top,\, b_r = 0 \quad &\text{for } r= i, \\
            W_{r,\cdot} = 0,\, b_r = 0 \quad &\text{for } r=V + 1,  \\
            W_{r,\cdot} = e_{r}^\top, \, b_r = 0 \quad &\text{otherwise.}
        \end{cases}
        \]
        \item Repeat the following layers $B+1$ times: \begin{enumerate}

            \item[($L_2$, $L_3$)] Assign $c \gets \mathbf{1}\set{x_i \leq x_e}$, a binary operator which requires two layers of (output) widths $V + 3$ and $V + 1$.
            \item[($L_4$)] For each variable $x \in \texttt{V} \backslash \set{x_i}$, create a temporary node in the neural network $x_{\text{old}}$ storing the current value of $x$. This layer has output width $\leq 2V$.
            \item[(Clause)]  To the variables in $\texttt{V} \backslash \set{x_i}$, apply $F_{\tC}$. Recall that the clause may not add variables or modify the loop counters. To the temporary nodes and $c$, apply the identity transformation. Explicitly, suppose $W$ and $b$ are the parameters of any layer in $F_{\tC}.$ Supposing the temporary nodes are indexed by the last $V$ variables, create a similar layer in $F_{\tP,N}$ that has parameters $\overline{W}$ and $\overline{b}$, given by
            \begin{equation}
            \label{eq:augmented_layer}
            \overline{W} = 
            \begin{bmatrix}
            W & 0 \\
            0 & I \\
            \end{bmatrix},\ \ 
            \overline{b} = 
            \begin{bmatrix}
            b \\
            0
            \end{bmatrix},
            \end{equation}
            where the $I$ is of order $V\times V$ and the $0$ vector in $\overline{b}$ has length $V$.
            \item[($L_5$, $L_6$)] Using the \texttt{if} construction, update each variable $x$ in $\tV\backslash \set{x_i}$ by 
                \[
                cx + (1-c)x_{\text{old}},
                \]
           simultaneously to all variables in $\tV\backslash \set{x_i}$. The \texttt{if} transformation creates two temporary variables for every variable to be updated. Hence the layer has width $\leq 4V.$ 
            \item[($L_7$)] Delete the temporary variable copies from ($L_4$), and set $x_i \gets x_i + 1$.  This layer has width $V + 1$.
        \end{enumerate}
        \item The final layer ($L_8$) deletes the variable $c$.
    \end{enumerate}

The encoding of the {\tt for} loop repeats the block of layers inside the {\tt for} loop $B+1$ times, but ensures that the clause is only applied $x_{e} - x_{s} + 1$ times, by keeping track of a counter variable. The advantage of this construction is that the layers applied to encode the {\tt for} loop are exactly the same copies of each other, repeated $B+1$ times. This is important so that the structure of the network does not depend so much on the input. 

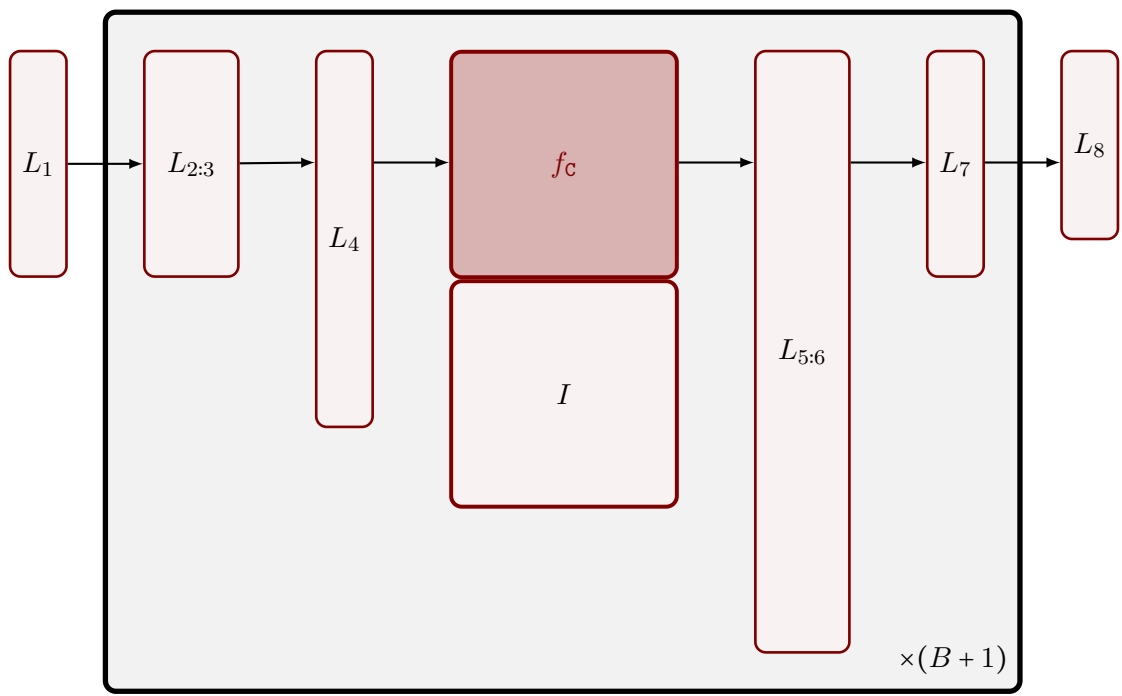
\begin{figure}[H]
\centering
\begin{tikzpicture}[
    align = center,
    box/.style={rectangle, draw, thick, rounded corners, minimum height=3cm, minimum width=0.75cm, anchor = north, draw = maroon, fill = maroon!5, line width = 1pt}
]


\node[box] (L1) at (0,0) {$L_1$};
\node[box, anchor=north west, minimum width = 1.25cm] (L2) at ($(L1.north east)+(1,0)$) {$L_{2:3}$};
\node[box, anchor=north west, minimum height = 5cm] (L4) at ($(L2.north east)+(1,0)$) {$L_4$};
\node[box, anchor=north west,
        color = maroon, fill = maroon!30, minimum width = 3cm, line width = 1.5pt] (clause) at ($(L4.north east)+(1,0)$) {$f_{\tC}$};
\node[box, anchor=north west, minimum width = 3cm, line width = 1.5pt] (id) at ($(clause.south west)+(0,0)$) {$I$};
\node[box, anchor=north west, minimum height = 8cm, minimum width = 1.25cm] (L5) at ($(clause.north east)+(1,0)$) {$L_{5:6}$};
\node[box, anchor=north west] (L7) at ($(L5.north east)+(1,0)$) {$L_7$};
\node[box, anchor=north west, minimum height = 2.5cm] (L8) at ($(L7.north east)+(1,0)$) {$L_8$};

\draw[->, thick] (L1) -- (L2);
\draw[->, thick] (L2) -- ($(L4.north west)+(0,-1.5)$);
\draw[->, thick] ($(L4.north east)+(0,-1.5)$) -- ($(clause.north west)+(0,-1.5)$);
\draw[->, thick] ($(clause.north east)+(0,-1.5)$) -- ($(L5.north west)+(0,-1.5)$);
\draw[->, thick] ($(L5.north east)+(0,-1.5)$) -- ($(L7.north west)+(0,-1.5)$);
\draw[->, thick] ($(L7.north east)+(0,-1.5)$) -- ($(L8.north west)+(0,-1.5)$);

\begin{pgfonlayer}{background}

\draw[box, thick, draw = black, fill = gray!10, line width = 2pt] ($(L2.north west)+(-0.5,0.5)$) rectangle ($(L5.south east)+(2.25,-0.5)$);

\node[below right, at={(L5.south east)}, xshift= +0.5cm, yshift=+0.25cm] (xB) {$\times (B+1)$};
\end{pgfonlayer}

\end{tikzpicture}
\caption{Schematic of the {\tt for} loop construction with clause $\tC$. Gray rectangle denotes a repetition of the layers contained within it $B+1$ times. }
\label{fig:forloop}
\end{figure}

\subsection{Proof of encoding}
The following theorem gives a formal proof that our scheme successfully encodes an SNP by a feedforward neural network.

\begin{theorem}
\label{thm:sneuralprograms_to_neuralnetworks}
Let $\tP$ be an SNP with variable context $\tV = (x_1,\dots,x_{V})$, indexed by the statements $(\tS_1,\dots,\tS_L)$. Let $\tP$ take in inputs $(x_1,\dots,x_I) \in [N]^{I}$ and be $B := B(N)$-bounded. Then for each $N,$ there is a feedforward neural network $F_{\tP,N}$ with ReLU nonlinearity, which agrees with the program for all inputs in $[N]^{I}.$ Further, all parameters of the neural network are bounded by $B$, and all {\tt for} loop layers in $\tP$ repeat $B+1$ times.
\end{theorem}

\begin{proof}
The proof proceeds by induction on the nested depth of SNPs with a fixed variable context $\tV$. For the base case, consider a program $\tP$ of nested depth $0$. By the conversion described in Section \ref{sec:base_case_conversion}, there is a neural network $F_{\tP,N}$ which exactly agrees with the output of $\tP$ for every choice of input in $[N]^I$, where the maximum parameter in the network is bounded by $B.$\footnote{Only the \texttt{if} layer construction has parameters that depend on $B$.}

For the inductive step, consider a program $\tP$ of nested depth $d$, with top level representation $(\tO_1,\dots,\tO_k).$ We must prove that the neural network defined by the composition
\[
f_{\tO_k,B} \circ f_{\tO_{k-1,B}} \dots \circ f_{\tO_1,B}
\]
agrees with the program $\tP$ for all inputs $\in [N]^{I}.$ If $\tO_i$ is an SNP statement, $f_{\tO_i, B}$ is equivalent to $\tO_i$ as a function $\bb{N}_0^V \rightarrow \bb{N}_0^V$. Consider the case where $\tO_i$ is a {\tt for} loop with clause $\tC$. The clause $\tC$ is also an SNP with variable context $\tV$, of nested depth equal to $d-1.$ By the inductive hypothesis, there exists a neural network $F_{\tC}$ on the variables $x_1,\dots,x_V$ which encodes the program $\tC$, agreeing on all values of the variables $x_1,\dots,x_V$ less than $B.$ The inductive hypothesis also guarantees all parameters of $F_{\tC}$ are bounded above by $B$, and all {\tt for} loop constructions in $\tC$ iterate at most $B+1$ times. The {\tt for} loop construction for $\tO_i$ with $B+1$ repetitions applies the clause $\tC$ to $\tV$ exactly $\edits{x_e - x_s} + 1$ times, since $\edits{x_e - x_s} + 1 \leq B+1$, and so agrees with $\tO_i$ as functions $\bb{N}_0^V \rightarrow \bb{N}_0^V$. When creating layers $(L_5,L_6)$ of $\tO_i$'s \texttt{for} loop construction, we use $B$ for the parameters of the \texttt{if} statements. This ensures all parameters in $F_{\tP,N}$ are bounded above by $B(N)$.
\end{proof}

\subsection{Maximum width of the neural network} 

The construction of $F_{\tP}$ has some additional properties which we record here. Firstly, the width of the neural network is controlled by the length of the SNP. Let $W_{\max}(F)$ be the maximum width of any feedforward neural network $F$.

\begin{lemma}[Bounding the maximum width of the neural network]
\label{lemma:maxwidth}
Consider an SNP $\tP$ with variable context $\tV$ of size $V$, length $L$, taking inputs $[N]^I$. Then
\[
W_{\max}(F_{\tP,N}) \leq 4VL.
\]
\end{lemma}

\begin{proof}
We prove this statement again by inducting on the nested depth $d$ of the program. The inductive claim will be $\maxWidth(\tP) \leq 4V\max(1,d)$. For the base case, consider $d = 0$ (so that there are no {\tt for} loops.) Then $\maxWidth(\tP) \leq V+3$ since there are $V$ variables in the program and all non-\texttt{for} loop SNP operations temporarily increase the width of the neural network by at most $3$.

Now, consider any program $\tP$ with length $L$ and maximum nested depth $d \geq 1$, and write its top-level representation as $(\tO_1,\dots,\tO_k).$ If $f_{\tO_i}$ denotes the sequence of neural network layers corresponding to $\tO_i$ in $F_{\tP,N}$, then
\[
W_{\max}(F_{\tP,N}) = \max_{i=1}^k(W_{\max}(f_{\tO_i}))
\]
If $\tO_i$ is an SNP statement, then $W_{\max}(f_{\tO_i}) \leq V+3$. Otherwise, consider when $\tO_i$ is a \texttt{for} loop with clause $\tC_i.$
Notice that $\tC_i$ is also an SNP, where the maximum nested depth is $d-1$. By the inductive hypothesis, $\maxWidth(F_{\tC_i}) \leq 4V\max(d-1,1)$. By inspecting the {\tt for } loop construction, we can bound the widths of the layers. Layers $L_1,L_2,L_3$ have widths at most $V+3$; $L_4, L_5,L_6$ have widths at most $4V$. The layers encoding the clause have widths at most $V + \maxWidth(F_{\tC_i})$, since $F_{\tC_i}$ is a mapping from $\bR^V \rightarrow \bR^V$. Finally, $L_7,L_8$ have widths at most $V+1$. As a result,
\begin{equation}
\label{eq:maxwidth_eq}
\maxWidth(f_{\tO_i}) \leq \max(V+3, 4V, V + \maxWidth(F_{\tC_i})) \leq  4V + \maxWidth(\tC_i)    
\end{equation}
 By equation~\eqref{eq:maxwidth_eq} and the inductive claim, we conclude that $\maxWidth(\tP) \leq 4V + 4V\max(d-1,1) \leq 4V\max(d,1)$. To deduce  the original claim, note that the maximum nested depth is at most~$L.$ 
\end{proof}

\subsection{Compressibility of the neural network} 

Secondly, the sequence of layers of the neural network $F_{\tP,N}$ are compressible, since \texttt{for} loops are encoded by repetitions of the same layers. To explicitly capture this, consider a $B$-bounded SNP $\texttt{P}$ with a fixed variable context $\tV$. We will define its \textit{repetition-compressed representation}, which will be a string using exponentiation to capture repetition of parameters. For example, if $\tP$ has a parameter representation $\theta_1\theta_2\theta_3\theta_3\theta_2\theta_3\theta_3$, we can express this as
\[
\theta_1(\theta_2(\theta_3)^2)^2
\]
where the two representations are equal when interpreted as words of the free algebra generated by all possible parameters.

To formally define the repetition-compressed representation of $\tP,$ first note that any SNP statement $\tS_i$ which is not a \texttt{for} loop  maps to a sequence of layers $f_{\tS_i,B} = g_{i,l_i} \circ \dots \circ g_{i,1}$. Each layer $g_{i,j}$ is parametrized by its weight matrix and bias vector $\theta_{i,j} = (W^{(i,j)},b^{(i,j)})$. Denote by $\Theta(f_{\tS_i,B})$ the sequence of parameters of the layers comprising $f_{\tS_i,B}$:
\[
\Theta(f_{\tS_i,B}) := \theta_{i,1}\theta_{i,2}\dots\theta_{i,l_i}, 
\]
interpreted as a word in the free algebra generated by all possible parameters. The \textit{repetition-compressed representation} of $\tP$, denoted $\cl{RC}(\tP)$, is defined inductively as follows. 

\begin{enumerate}
    \item  \textit{Base case.} Consider any program $\tP = (\tS_1,\dots,\tS_L)$ of nested depth $0$, so that there are no \texttt{for} loops. Define its compressed representation as
        \[
        \prod_{i = 1}^{L} \Theta(f_{\tS_i,B}),
        \]
    the concatenation of $\Theta(f_{\tS_i,B})$ for all $i$. This is also the same as $\Theta(F_{\tP,N}).$
    
    \item  \textit{Inductive step.} Now, consider any SNP of nested depth $d \geq 1$. Denote the top-level representation of $\tP$ by the sequence $(\tO_1,\dots,\tO_k)$. If $\tO_i$ represents the \texttt{for} loop statement $\tS_j$ with clause $\tC$, extend the map $\Theta$ by
        \[
        \Theta(\tO_i) = \theta_{i,1}(\theta_{i,2}\theta_{i,3}\theta_{i,4}\overline{\cl{RC}(\tC)}\theta_{i,5}\theta_{i,6}\theta_{i,7})^{B+1}\theta_{i,8}
        \]
    where $\theta_{i,\cdot}$ denote the layer parameters in the \texttt{for} loop construction, and $\overline{\cl{RC}(\tC)}$ is the repetition compressed representation, replacing every parameter $\theta = (W,b)$ with its augmented version $\overline{\theta} := (\overline{W},\overline{b})$ as in Eq.~\eqref{eq:augmented_layer}. Finally, define $\cl{RC}(\tP)$ to be the concatenation $\prod_{i=1}^k \Theta(\tO_i).$ 
\end{enumerate}

\begin{example}
Consider the following program, which has maximum bound $B(N) \leq 11.$ The variable context is:
\begin{minted}{python}
int i = 1
int j = 1
int res = 0
\end{minted}
The statements of the program are as follows:
\begin{minted}[linenos]{python}
for i = 1,...,10:
    res = 0
    for j = 1,...,10:
        res = res + 1
return res
\end{minted}
The program has only one top level \texttt{for} loop, on line $1$, with clause $\tC$ consisting of lines 2-4. It can be written as $\tO_1\tO_2$ where $\tO_1$ represents the \texttt{for} loop on line $1$ with its clause $\tC$, and $\tO_2$ represents $\tS_5$.
Then 
\[
\cl{RC}(\tC) = \theta_{2,1}\theta_{3,1}(\theta_{3,2}\theta_{3,3}\theta_{3,4}\overline{\theta}_{4,1}\theta_{3,5}\theta_{3,6} \theta_{3,7})^{B+1}\theta_{3,8}. 
\]
where $\overline{\theta}_{4,1}$ is the parameter representation of $\tS_4$, augmented by the \texttt{for} loop construction of line $3.$ Altogether,
\[
\cl{RC}(\tP) = \theta_{1,1} (\theta_{1,2}\edits{\theta_{1,3}}\theta_{1,4}\overline{\cl{RC}(\tC)}\theta_{1,5}\theta_{1,6}\theta_{1,7})^{B+1} \theta_{1,8} \theta_{5,1}.
\]
\end{example}
The main claim is that the resulting string, when interpreted as an element of the free algebra generated by all possible parameters $\theta$, is equal to the full parameter sequence of $F_{\tP,N}.$

\begin{proposition}[The layers of $F_{\tP,N}$ are efficiently describable]
\label{prop:nn_compressed_representation}
Consider an SNP $\tP$ of length $L$ with variable context $\tV$, bounded by $B := B(N)$, with inputs in $[N]^I$. Denote its neural network encoding by $F_{\tP,N}$. Let $\cl{RC}(\tP)$ denote the repetition-compressed layer representation of the SNP $\tP.$ Then:
\begin{itemize}
    \item $\cl{RC}(\tP)$ is equivalent to the sequence of parameters of $F_{\tP,N}.$ 
    \item The number of unique symbols $\theta$ in $\cl{RC}(\tP)$ is $\le 8L$.
    \item The number of parenthesis pairs $(\dots)^{B+1}$ in $\cl{RC}(\tP)$ is equal to the number of \textup{\texttt{for}} loops in $\tP$.
\end{itemize}

\end{proposition}
The first claim is evident from the induction and \texttt{for} loop construction in Theorem \ref{thm:sneuralprograms_to_neuralnetworks}. When $\tP$ is a depth zero program, $\cl{RC}(\tP)$ is exactly equal to the parameter sequence of $F_{\tP,N}.$ In the general case, consider a program $\tP$ of depth $d > 1$ with top-level representation $(\tO_1,\dots,\tO_k)$; then $\prod_{i=1}^k \Theta(\tO_i).$ If $\tO_i$ is a \texttt{for} loop with clause $\tC_i$, $\Theta(\tO_i)$ exactly encodes the elements of the \texttt{for} loop construction: (1) the $8$ additional layers in the \texttt{for} loop  construction, (2) the repetition of layers $B+1$ times, and (3) the augmenting of layers corresponding to $\tC_i$. 
\begin{proof}[Proof of Proposition \ref{prop:nn_compressed_representation}]
The proof of the second and third properties follows from induction on the depth of a program $\tP$. For the base case, consider a program $\tP$ of depth zero. In this case, the number of parameter symbols $\theta$ in $\cl{RC}(\tP)$ is at most $2L$, since every non \texttt{for} loop statement can be encoded in at most two layers. There are no \texttt{for} loops or parentheses in $\cl{RC}(\tP)$. This establishes the base case. 

For the inductive step, consider a program $\tP$ of depth $d > 1$ with top-level representation $(\tO_1,\dots,\tO_k).$ Recall that $\cl{RC}(\tP)$ is the concatenation $\prod_{i=1}^k \Theta(\tO_i).$ To show the second property, let $u(S)$ be the number of unique $\theta$ symbols in a string $S$ in the free algebra generated by all parameter values. Then the total number of unique symbols in $\cl{RC}(\tP)$ is at most
\[
\sum_{i=1}^k u(\Theta(\tO_i)).
\]
If $\tO_i$ is an SNP statement, then $ u(\Theta(\tO_i)) \leq 2$ as observed in the base case. If it is a \texttt{\texttt{for} loop} with clause $\tC_i$, the number of symbols is $8 + u(\cl{RC}(\tC))$, since the \texttt{\texttt{for} loop} construction creates $8$ additional layers. In this case, the inductive hypothesis gives $ u(\Theta(\tO_i)) \leq 8(\text{length}(\tC_i) + 1).$ The number of top-level statements plus the sum of lengths of all top-level clauses is equal to $L$, proving that $u(\cl{RC}(\tP)) \leq 8L$. A similar argument shows that the number of parenthesis pairs in $\cl{RC}(\tP)$ is equal to the number of \texttt{for} loops in the program.
\end{proof}

\section{A Measure of Description Length for Neural Networks}
\label{sec:description_complexity}

In this section, we introduce a description length measure for the encoded neural network. The measure roughly corresponds to the number of symbols needed to describe the parameters of the neural network. We will use the following alphabet $\scr{A}$ of symbols:
\begin{enumerate}
    \item A symbol $\cl{I}$ to represent a node which is an input into the neural network.
    \item $,$ to mark the start of a new number.
    \item $-,0,1$ to describe binary expansions of integers.
    \item $(\dots)^{*k}$ to describe $k$ fold repetition of a substring of symbols, with $k$ encoded in binary.\footnote{$*k$ is written as a superscript only for clarity; there is no distinction between symbols which are in superscript and those in normal font.}
    \item Symbols $\cl{W},\cl{B}$ to demarcate the weight matrix and the bias vector of a layer: following the symbols are the values of the weights and biases.
\end{enumerate}
There are a finite number of symbols in $\scr{A}$ given by $\set{\cl{I}, ``,",\edits{-,}0,1,*,\cl{W},\cl{B},``(",``)"}$. Every feedforward neural network can be converted to a sequence of symbols, by specifying the weights and biases of every layer using the symbols above. Let $\bin(n)$ for $n \in \bb{N}_0$ be the binary expansion of a number.

\begin{definition}[Full symbol encoding of a neural network]
\label{def:symbol_encoding}
Given a neural network $F$ of depth $d$, let $(\theta_1,\dots,\theta_d)$ be the sequence of parameters of the layers, which can be rewritten as 
\begin{equation}
\label{eq:parameter_seq}
(W^{(1)},b^{(1)},W^{(2)},b^{(2)},\dots,W^{(d)},b^{(d)}).
\end{equation}
Suppose the input to $F$ is a vector $\mathbf{x} \in \bR^n$ where some coordinates of $\mathbf{x}$ may be fixed, and some may be free variables. Define the string $S_1$ by replacing
\begin{enumerate}
    \item each weight matrix $W$ in Eq. \eqref{eq:parameter_seq} by its vectorization in binary, prefixed with the $\cl{W}$ symbol:
    \[
    \cl{W},\bin(W_{1,1}),\bin(W_{1,2}),\dots,\bin(W_{m,n}),
    \]
    where $W \in \mathbb{Z}^{m \times n}$, and 
    \item each bias vector by the symbol $\cl{B}$ and its entries encoded in binary, separated by commas:
    \[
    \cl{B},\bin(b_1),\dots,\bin(b_m),
    \]
    assuming $b \in \mathbb{Z}^m$.
\end{enumerate}
Secondly, define the string $S_2$ by replacing all free variables in the input vector $\mathbf{x}$ by the symbol $\cl{I}$, and the other coordinates by their binary representations. Then, the $\scr{A}$-symbol sequence encoding $F$ is the string obtained by concatenating $S_2$ followed by $S_1$.
\end{definition}

\begin{example}
Consider a neural network with input $\bR^2$ with two layers, 
\[
W^{(1)} = 
\begin{bmatrix}
1 & 1 \\
1 & 1
\end{bmatrix},
b^{(1)} = 
\begin{bmatrix}
5 \\
5
\end{bmatrix},
\quad 
W^{(2)} = 
\begin{bmatrix}
3 & 1 \\
\end{bmatrix},
b^{\edits{(2)}} = 
\begin{bmatrix}
2 \\
\end{bmatrix},
\]
operating on the vector $[x,1]$. The full symbol sequence associated with the neural network is 
\[
\cl{I},1,\cl{W},1,1,1,1,\cl{B},101,101,\cl{W},11,1,\cl{B},10.
\]
A shorter symbol sequence describing the same network is
\[
\cl{I},1,\cl{W},(1,)^{*100}\cl{B},(101,)^{*10}\cl{W},11,1,\cl{B},10.
\]
\end{example}

Conversely, not every sequence of symbols corresponds to a neural network. A symbol sequence $S$ is called \textit{valid} if after expanding $k$-fold repetitions of substrings to obtain $\tilde{S}$, there exists a neural network whose symbol description equals $\tilde{S}.$ In other words, a sequence of symbols is valid if one can define a sequence of neural network layers by reading off from the string.

\begin{definition}
Given a neural network $F$, a sequence of symbols $a_1,\dots,a_S$ in $\scr{A}$ \textit{describes} $F$ if the neural network generated by the sequence $(a_1,\dots,a_S)$ is exactly equal to $F$.\footnote{In the sense that the layer dimensions and parameters of the neural networks must agree} The \textit{description length} of $F$ is the minimum length over symbol sequences which describe $F$. 
\end{definition}

With this definition of description length, we can show that the neural network encoding a simple neural program $\tP$ has a description length controlled by the length of $\tP$ and the number of variables.

\begin{proposition}
\label{prop:efficient_conversion}
Consider a length $L$ SNP $\tP$ which takes inputs in $[N]^I$, is $B(N)$ bounded, and has variable context $\tV$ of size $V$. Then $F_{\tP,N}$ has description length at most $cL^3V^2\log_2B(N)$ for some universal constant $c$. 
\end{proposition}

\begin{proof}
Consider the repetition-compressed representation of $F_{\tP,N}$, $\cl{RC}(\tP)$. By Proposition \ref{prop:nn_compressed_representation}, replacing every parameter instance $\theta$ in $\cl{RC}(\tP)$ by its alphabet description as in Definition \ref{def:symbol_encoding}, and every parenthesis $(\dots)^{B+1}$ with $(\dots)^{*\bin(B+1)}$, results in a symbol sequence which describes $F_{\tP,N}.$ 

By Lemma \ref{lemma:maxwidth}, the maximum width of the neural network $F_{\tP,N}$ is $\le 4LV$. By Theorem~\ref{thm:sneuralprograms_to_neuralnetworks}, the maximum number appearing in the weights and biases of the encoded neural network is at most $\edits{2}B(N)$, which takes at most $\edits{O}\left(\log_2B(N)\right)$ symbols to encode. Each weight matrix has at most $O(L^2V^2)$ entries, each of which takes $\log_2B(N)$ $\scr{A}$-symbols to encode, while every bias vector requires at most $O(LV\log_2B(N))$ $\scr{A}$-symbols to encode. Thus every parameter symbol $\theta$ can be encoded in $O(L^2V^2\log_2 B(N))$ many $\scr{A}$-symbols. Furthermore, Proposition \ref{prop:nn_compressed_representation} shows there are at most $8L$ parameter symbols $\theta$ and $O(L)$ many parenthesis pairs. Each parenthesis pair contributes $O(1) + \log_2B(N)$ many $\scr{A}$-symbols to the description length, while the total number of $\scr{A}$-symbols to encode the parameter symbols is $O(L^3V^2\log_2B(N))$, leading to the desired bound.
\end{proof}

\begin{lemma}
\label{lemma:exponentially_many_sparse_nns}
Let $\cl{N}_K$ be the set of neural networks of description length at most $K$. Then $|\cl{N}_K| \leq e^{cK}$ where $c$ is a universal constant. 
\end{lemma}
\begin{proof}

For a given neural network in $\cl{N}_K$, assign to it a shortest valid symbol sequence which describes it, of length at most $K.$ A valid symbol sequence describes exactly one neural network. Thus, there is an injection from $\cl{N}_K$ to valid symbol sequences of length at most $K$. The number of valid symbol sequence of length at most $K$ is less than $e^{cK}$ for a universal constant $c$, as there are only a finite number of symbols in the alphabet.
\end{proof}

\begin{remark}[Tightness of Compression]
We expect that the bound in Proposition \ref{prop:efficient_conversion} may be tightened in $L,V$ because most weight matrices are sparse: the identity operation is applied to most of the variables, and each layer modifies a constant number of variables. However this comes at the expense of complicating the analysis without changing the core idea of a complexity bound polynomial in attributes of $\tP$. For a start at a lower bound, we can show that there exists a program where the description length of the network is at least $O(L \log_2 B)$, arguing by the probabilistic method. Let $n \in \mathbb{N}$ be fixed. Consider a random SNP $\tP$ of one variable $x$. There are $L$ lines in the SNP; line $j$ of the program sets $x$ to be the number $M_j := \sum_{i=1}^n c_i^{(j)} 2^i$ for $c^{(j)}_1,\dots,c^{(j)}_n$ an i.i.d. random collection of Bernoulli$(1/2)$ random variables, and returns $x$. In this program, the symbol encoding of the neural network $F_{\tP}$ from Definition \ref{def:symbol_encoding} contains as a subsequence the sequence $(c_i^{(j)})_{i=1,\dots,n,j=1,\dots,L}$. By standard incompressibility arguments based on counting \cite[Theorem 2.2.1]{li2008introduction}, most random sequences are not compressible by more than a constant factor. Thus, there exists a program $\tP$ where the neural network $F_{\tP}$ cannot have description length less than a constant factor times $Ln = L \log_2 B$.

\end{remark}

\section{Neural Networks Generalize on Data from Short Programs}
\label{sec:generalization}
The following result is our main theorem. Roughly, it says that any minimum description length interpolator generalizes on low complexity data.

\begin{theorem}
\label{thm:generalization_bound}
Let $\tP$ be an SNP of length $L$ which outputs a result $\tP(x)$ for each input $x \in [N]^I$, with maximum bound $B(N)\ge 2$. Let $\tP$ have $V$ variables, with $V \geq I.$ Fix $\e >0, \delta \in (0,1)$ and let
\[
n = \frac{c_3L^3V^2 \ln B(N) + \ln \frac{1}{\delta}}{\e},
\]
for an absolute constant $c_3.$ Suppose we observe i.i.d.~data $(X_i,Y_i),i=1,\dots,n$ where $X_i$ is drawn i.i.d. from a distribution $\mu$ on $[N]^I$, and $Y_i = \tP(X_i).$ Let $\widehat{f}_{\MDL}$ be a minimum-description length neural network interpolating the data. Then with probability $\ge 1 - \delta$, the error rate of $\widehat{f}_{\MDL}$ on a test point drawn from $\mu$ is at most $\e.$
\end{theorem}

\begin{proof}
Throughout this proof, $c_0, c_1,\ldots$ will denote positive universal constants. By our previous results, there exists a neural network $F_{\tP}$ of description length $\le s:= c_0L^3V^2\log_2B(N)$ which encodes the program $\tP.$ Letting $\cl{N}_s$ be the set of all neural networks with description length $\le s$, Lemma \ref{lemma:exponentially_many_sparse_nns} states that $|\cl{N}_s| \leq e^{c_1 s} \le B(N)^{c_2L^3V^2}.$

Take any two networks $f_1,f_2 \in \cl{N}_s$ which disagree on a subset of $[N]^I$ with $\mu$-measure $\geq \e$. The chance that $f_1,f_2$ agree on the data is $\leq (1-\e)^n$, where recall $n$ is the number of data points. Let $A$ be the event that there exist $f_1,f_2 \in \cl{N}_s$ which disagree on a subset $\mu$-measure $\geq \e$ but agree on the data. By the previous point,
\[
\Prob(A) \leq \binom{|\cl{N}_s|}{2} (1-\e)^n.
\]
Now, consider $F_\tP$ and $\widehat{f}_{\MDL}$, the minimum description length neural network which interpolates the data.\footnote{This always exists as $F_{\tP}$ interpolates the data} Both of these are in $\cl{N}_s$, as $\widehat{f}_{\MDL}$ must have description length less than or equal to the description length of $F_\tP$, and they both agree on the observed data. On the event $A^c$, $\widehat{f}_{\MDL}$ and $F_\tP$ will agree on a subset $S$ with $\mu(S) \geq 1-\e$, so they will agree on a new test point with probability $\geq 1 - \e.$ Now, from the previous display, we get
\[
\Prob(A) \leq \frac{1}{2}|\cl{N}_s|^2e^{-n\e} \le e^{c_3L^3V^2\ln B(N) - n\e}.
\]
Plugging in $n$, we conclude that $\Prob(A) \leq \delta$ as desired.
\end{proof}

{

The same proof gives an averaged generalization guarantee that is much simpler.

\begin{corollary}
\label{cor:avg_gen}
Consider a SNP $\tP$ with the same conditions as in Theorem \ref{thm:generalization_bound}, and a dataset $\set{(X_i,Y_i)}_{i=1}^n$, with $X_i$ iid from a distribution $\mu$ on $[N]^I$ and a generic $n$. Let $\widehat{f}_{\MDL}$ be a minimum-description length neural network interpolating the data. Then for a new sample $x \sim \mu$, 
\[
\Prob\left( \widehat{f}_{\MDL}(x) \neq \tP(x) \right) \leq \frac{CL^3V^2 \ln B(N)}{n}
\]
for an absolute constant $C$.
\end{corollary}
\begin{proof}
Fix the same notation as in the proof of Theorem \ref{thm:generalization_bound}. Consider again the event $A$; by the previous proof's logic, on the event $A^c$, $\widehat{f}_{\MDL}$ and $F_\tP$ agree on a new test point with probability $\geq 1 - \e.$ Taking the contrapositive, we have that 
\[
\Prob\left( \widehat{f}_{\MDL}(x) \neq \tP(x) \ | \ \cl{D} \right) \geq  \e \Rightarrow A,
\]
where $\cl{D}$ denotes the dataset $\set{(X_i,Y_i)}_{i=1}^n.$ Taking unconditional probabilities, we see that 
\begin{align*}
\Prob\left(\Prob\left( \widehat{f}_{\MDL}(x) \neq \tP(x) \ | \ \cl{D} \right) \geq  \e \right) & \leq \Prob(A) \\
                        & \leq 1 \wedge e^{c_3L^3V^2\ln B(N) - n\e}.
\end{align*}
Now, integrating over $\e$ from $0$ to $1$ on the left hand side yields the unconditional error 
\[
\Prob\left( \widehat{f}_{\MDL}(x) \neq \tP(x)\right).
\]
Similarly on the right side, splitting into integrals over the intervals $[0,c_3L^3V^2\ln B(N)/n]$ and $[c_3L^3V^2\ln B(N)/n,1]$, yields an upper bound of $\left(1 + c_3L^3V^2\ln B(N) \right)/n$. Thus, we have 
\[
\Prob\left( \widehat{f}_{\MDL}(x) \neq \tP(x) \right) \leq \frac{CL^3V^2 \ln B(N)}{n}
\]
as desired, for some absolute constant $C$.
\end{proof}
}

\begin{remark}
The approach of Theorem \ref{thm:generalization_bound} can be extended to derive guarantees on any neural network interpolator with description length $s'$ slightly greater than the description length of $F_{\tP}$. The details are given in the proof of Theorem \ref{thm:generalization with noise}, which separately extends Theorem \ref{thm:generalization_bound} to noisy data. A guarantee of this type may be useful if one knew the program $\tP$ or had upper bounds on its complexity. The advantage of $\widehat{f}_{\MDL}$ is that it is agnostic to the structure of the data; we require no information on the program $\tP.$
\end{remark}

\section{Examples}\label{examples}
In the first two examples below, $N$ is a large number, and our data consists of $(x_i,y_i)$, $i=1,\ldots,n$, $x_1,\ldots,x_n$ are drawn i.i.d.~uniformly from $[N] := \{1,\ldots,\edits{N}\}$, and $y_i = f(x_i)$ for some given function $f$. We will fix $\e,\delta$ and apply Theorem \ref{thm:generalization_bound} to determine the number of training samples $n$ needed to achieve $\e$ test error rate with $1-\delta$ probability.

\begin{example}[Prime Numbers]
Let us revisit the prime checking program in the introduction. Here, $f(x)=1$ if $x$ is prime and $0$ if not. The full SNP may be found in Example \ref{ex:primenumber_program_full}; the program satisfies $L = 11, V = 9, B(N) = N^2$. Let $\delta = 0.01$ and $\e = N^{-\nu}$ for any $\nu \in (0,1)$. Recall that the density of the primes among the first $N$ natural numbers is $(\ln N)^{-1}$ via the prime number theorem. Therefore $\widehat{f}$ would classify both primes and non-primes correctly with high accuracy. By Theorem~\ref{thm:generalization_bound}, the MDL interpolating neural network requires on the order of $O\left( N^{\edits{\nu}}\left( 2\ln 10 + \ln N\right)\right)$  training samples to achieve test error less than $\e$ with probability at least $0.99$. The big-O notation hides constant factors depending on $L,V$ and the absolute constants $c_3$ in Theorem \ref{thm:generalization_bound}. Corollary \ref{cor:avg_gen} also gives an averaged generalization guarantee. On randomly sampled datasets of size $n$, we have
\[
\Prob(\widehat{f}_{\MDL}(x) \neq \tP(x)) \leq \frac{C \ln N}{n}.
\]
Thus $\widehat{f}$ would classify primes and non-primes accurately if $n \gg (\ln N)^2$, in an averaged sense.

\end{example}

\begin{example}[Sums of Squares]
Let $f(x)=1$ if $x$ is a sum of two squares and $0$ if not. This is easily expressed as a composite SNP $\tP_{\textsf{SOS}}$:
\begin{minted}{python}
input n
int i = 0
int j = 0
int res = 0
int square1 = 0
int square2 = 0
bool output = 0
bool sum_of_squares = 0
for i = 0,...,n:
    for j = 0,...,n:
        square1 = multiply(i,i)
        square2 = multiply(j,j)
        sum_of_squares = (square1 + square2 == n)
        res = res + sum_of_squares
output = (res > 0)
return output
\end{minted}
From the full atomic program written out in Example \ref{ex:sos_full}, we have $L = 13, V = 11, B(N) = 2N^2$. By Theorem \ref{thm:generalization_bound}, 
for $\delta = 0.01$ and $\e = N^{-\nu}$ for any $\nu \in (0,1)$, the MDL interpolating neural network $\widehat{f}$ requires $n = O\left( N^{-\nu}\left( 2\ln 10 + \ln N\right)\right)$ many samples, like in the previous example, to obtain test error less than $\e$ with probability greater than $1 - \delta$. A result of Landau \cite{landau1909einteilung} says that the number of integers less than $N$ which can be expressed as a sum of two squares asymptotically scales like $K N/\sqrt{\ln N}$ with a known formula for the constant $K.$ Thus, $\widehat{f}$ identifies both sums of squares and non-sums of squares accurately. On generic datasets of size $n$, Corollary \ref{cor:avg_gen} shows that the probability $\widehat{f}(x)$ is incorrect on a fresh sample is also of order $\ln N / n$; therefore $\widehat{f}(x)$ would classify correctly with non-trivial accuracy if $n \gg (\ln N)^{3/2}.$
\end{example}

In the next example, the $\mathbf{x}_i$ are vectors drawn from uniformly $[N]^3$, and $y_i = f(\mathbf{x}_i)$. 

\begin{example}[Sides of triangles]
Given a triple of nonnegative integers $(x_1,x_2,x_3)$ the following program checks whether these can be the side lengths of a triangle: 
\begin{minted}{python}
input x1
input x2
input x3
int temp = 0
bool check = 0
bool res = 0
int s = 0
temp = x1 + x2
check = (temp > x3)
s = s + check 
temp = x2 + x3
check = (temp > x1)
s = s + check 
temp = x1 + x3
check = (temp > x2)
s = s + check 
res = (s == 3)
return res
\end{minted}
With inputs in $[N]^3$, this is an SNP with $V = 7, L = 11, B(N) = 2N$. By a volumetric argument, the asymptotic number of triples $(x_1,x_2,x_3) \in [N]^3$ which are sides of a triangle is $1/2$. As with the previous examples, for a fixed $\delta = 0.01$ say and desired error level $\e$, we require $n = \Omega(\frac{\ln N + \ln\frac{1}{\delta}}{\e})$ many samples, with constants depending on $L,V,c_3.$ For any $\e < 1/2$ the resulting error rate is better than random guessing. Corollary \ref{cor:avg_gen} in this case requires general datasets with $n \gg \ln N$ for $\widehat{f}$ to classify both cases with non-trivial accuracy.
\end{example}

\begin{remark}[Non-uniform input distribution]
The previous examples in this section used the uniform distribution on $[N]^I$ for $\mu$. The strength of Theorem \ref{thm:generalization_bound}'s guarantees should be interpreted in light of properties of $\mu$. For example, consider the prime checking SNP $\tP$ and consider an input distribution $\mu$ where $\supp \mu$ is contained in the set of non-prime numbers. In this case, $\tP(X_i)$ agrees with the simpler program of outputting only zero, and so the guarantee of Theorem \ref{thm:generalization_bound} can be strengthened. Notice that the minimum description length interpolator is agnostic to the fact that a simpler program exists which interpolates the data; this is a form of adaptivity. In a separate direction, suppose $\supp \mu$ is some subinterval of $[N]$, not growing with $[N]$. Then clearly generalization is trivial if $n$ is on the order of $\supp \mu$, since all future test examples are likely to be contained in the training data.
\end{remark}

\section{Extension to Noisy Data}
\label{sec:noisy_data}

Theorem \ref{thm:generalization_bound} extends to noisy or corrupted data. The idea is as follows. Suppose the true dataset is generated as the output of a fixed neural network $F$. If the corrupted dataset can be interpolated by another neural network, not much more complex than $F$, then the proof strategy of Thm. \ref{thm:generalization_bound} continues to hold. We make this concrete under a model of sparse noise. That is, if the amount of noisy labels is small then we can hard code the noisy labels in a smaller neural network, and adjoin it to $F$. The next lemma makes this clear.

\begin{lemma}
\label{lemma:noisy_perturbation_nn}
Let $F$ be a neural network with description length $s$, such that for all inputs $x \in [N]^I$ the output is $F(x) \in \mathbb{Z}_{\geq 0}$. Consider another function $\tilde{F}:[N]^I \rightarrow \mathbb{Z}_{\geq 0}$ and let $E := \set{x \in [N]^I: \tilde{F}(x) \neq F(x)}$. Suppose further that $\tilde{F}(x), F(x) \leq B$ for all $x \in [N]^I$ for some constant $B.$ Then there exists a neural network $G$ with description length $s + O\left(\edits{I^2}|E|\log_2(B+I) \right)$ that agrees with $\tilde{F}$ on all of $[N]^I$.
\end{lemma}
\begin{proof}

\edits{
We show how to augment the neural network $F$ to create a neural network $G$ that agrees with $\tilde{F}$ by appending a small number of additional layers to $F$. The neural network $G$ adds $F$ and the function
\[
g(x) := \sum_{y \in E} (\tilde{F}(y)- F(y))\mathbf{1}\set{x = y}.
\]
We will encode $g$ as a neural network with $O(I)$ hidden width and $O(|E|)$ layers. Enumerate the elements in $E$ as $y^{(1)},\dots,y^{(|E|)}$. The neural network will sequentially calculate $(\tilde{F}(y^{(j)}) - F(y^{(j)}))\mathbf{1}\set{x = y^{(j)}}$ for $j = 1,\dots,|E|$ and cumulatively add the result onto the node containing $F(x)$. To do this, note that for any $j,$ $\mathbf{1}\set{x = y^{(j)}}$ may be calculated in $O(1)$ many layers of width $O(I)$. Indeed, by Eq. \eqref{eq:equality_check} we may use two layers which create $(\mathbf{1}\set{x_i = y^{(j)}_{i}})_{i = 1,\dots,I}$ followed by two layers which calculate $\mathbf{1}\set{x_i = y^{(j)}_{i}, i=1,\dots,I}$ using the relation
\[
\mathbf{1}\set{x_i = y^{(j)}_i, i=1,\dots,I} = \mathbf{1}\bigg \lbrace\sum_{i=1}^I \mathbf{1}\set{x_i = y^{(j)}_i} \geq I\bigg \rbrace.
\]
Finally, we use one last layer which adds $(\tilde{F}(y^{(j)}) - F(y^{(j)}))\mathbf{1}\set{x_i = y^{(j)}_{i}, i=1,\dots,I}$ to the node $F(x).$ The parameters within each of these layers are bounded above by $\max(B,I)$, from which we obtain a description length of $O(I^2 |E| \log_2 (B+I))$ for the neural network corresponding to $g.$ 
}

\end{proof}

\begin{theorem}[Generalization with Corrupted Data]
\label{thm:generalization with noise}Let $\tP$ be an SNP of length $L$ which outputs a result $\tP(x)$ for each input $x \in [N]^I$, with maximum bound $B(N)\ge 2$. Let $\tP$ have $V$ variables, with $V \geq I.$
Consider a dataset $(X_i,Y_i)_{i=1}^n$ generated by a SNP $\tP$ with $X_i \sim \operatorname{Unif}[N]^I$ and $Y_i = \tP(X_i)$. Let $(X_i,\tilde{Y}_i)$ be a corrupted version of the dataset, with $\tilde{Y}_i \neq \edits{Y_i}$ for at most $\rho n$ many indices $i.$ Suppose further that $\tilde{Y}_i \in \mathbb{N}$ and $\tilde{Y}_i \leq B(N).$ Let $\widehat{f}_{\MDL}$ be the minimum-description length neural network interpolating the corrupted data. Then with probability greater than
\[
1 - e^{c_3\left(L^3V^2 + \edits{I^2} \rho n \right)\ln (I+B(N)) - n\e^2/(2\rho + 2\e)},
\]
the error rate of $\widehat{f}_{\MDL}$ on a newly chosen test point is at most $\rho + \e$. Furthermore, we have
\[
\Prob_{x, \cl{D}}(\widehat{f}_{\MDL}(x) \neq \tP(x)) = C\rho + O\left(\frac{1}{n}\right),
\]
for an absolute constant $C = 1 + 2c_3\edits{I^2}\ln(I+\edits{B}(N)) + \sqrt{2c_3\edits{I^2}\ln(I+\edits{B}(N))}.$
\end{theorem}

This result extends Theorem \ref{thm:generalization_bound} to noisy data, with an arbitrary noise distribution and correlation, as long as the noise is sparse. The result shows generalization conditional on a realization of noise.
In this case, the minimum description-length interpolator generalizes neither optimally nor poorly, but rather displays \textit{tempered} overfitting \cite{manoj2023interpolation, harel2024provable}. For simplicity, we restrict to the case of a uniform input distribution, but expect that non-uniform input distributions can be handled with more refined concentration inequalities.

\begin{proof}
By our previous results, there exists a neural network $F_{\tP}$ of description length $\le s:= c_0L^3V^2\log_2B(N)$ which encodes the program $\tP.$ By Lemma \ref{lemma:noisy_perturbation_nn} there exists another neural network $F_{\textsf{corr}}$ of description length 
\[
s' = s + O\left(\rho \edits{I^2} n \ln(I + B(N))\right)
\]
that interpolates the training data. Letting $\cl{N}_{s'}$ be the set of all neural networks with description length $\le s'$, Lemma \ref{lemma:exponentially_many_sparse_nns} states that 
\[
|\cl{N}_{s'}| \leq e^{c_1 s'} \le B(N)^{c_3L^3V^2} e^{c_3 \rho \edits{I^2} n (\ln(I + B(N)))}.
\]
Now, let $A$ be the event that there exist $f_1,f_2 \in \cl{N}_{s'}$ which disagree on at least $\left(\rho + \e\right) N^I$ points but agree on at least $1-\rho$ fraction of the corrupted training data. By a union bound and Lemma \ref{lemma:binomial_tail},
\begin{align*}
\Prob(A) & \leq \binom{|\cl{N}_{s'}|}{2} \Prob\left(\textsf{Bin}(n,\rho + \e) \leq n\rho \right) \\
& \leq \binom{|\cl{N}_{s'}|}{2} \exp(-\frac{n\e^2}{2(\rho + \e)}).
\end{align*}
Next, consider $F_\tP$ and $\widehat{f}_{\MDL}$, the minimum description length neural network which interpolates the corrupted data. Both of these are in $\cl{N}_{s'}$, as $\widehat{f}_{\MDL}$ must have description length less than or equal to the description length of $F_{\textsf{corr}}$. Further $F_{\tP}$ and $\widehat{f}_{\MDL}$ both agree on $1-\rho$ fraction of the corrupted data. On the event $A^c$, $\widehat{f}_{\MDL}$ and $F_\tP$ will agree on $(1-\e- \rho)N^I$ points, so they will agree on a uniformly chosen test point with probability $\geq 1 - \e - \rho.$ Now, from the previous display, we get
\[
\Prob(A) \leq \frac{1}{2}|\cl{N}_s|^2e^{-n\e^2/\edits{2(\rho + \e)}} \le e^{c_3\left(L^3V^2 + \rho \edits{I^2}n \right)\ln (I +B(N)) - n\e^2/(2\rho + 2\e)}.
\]
This concludes the proof of the first statement. For the second, we emulate Corollary \ref{cor:avg_gen} and exploit the identity
\begin{equation}
\label{eq:integration_by_parts_avg_gen}
\Prob_{x,\cl{D}}(\widehat{f}_{\MDL} \neq \tP(x)) = \int_{0}^1 \Prob\left( \Prob(\widehat{f}_{\MDL}(x) \neq \tP(x) \ | \ \cl{D}) \geq x \right)  dx.
\end{equation}
By our previous argument conditional on the dataset $\cl{D}$ we have 
\[
\Prob\left(\Prob(\widehat{f}_{\MDL}(x) \neq \tP(x) \ | \ \cl{D}) \geq \rho + \e\right)  \leq 1 \wedge e^{c_3\left(L^3V^2 + \rho \edits{I^2}n \right)\ln (I +B(N)) - n\e^2/(2\rho + 2\e)}.
\]
Let $\e^*$ be the zero of the function $\e \mapsto c_3\left(L^3V^2 + \rho \edits{I^2}n \right)\ln (I +B(N)) - n\e^2/(2\rho + 2\e)$. Splitting the integral in Eq. \eqref{eq:integration_by_parts_avg_gen} into $[0,\rho],[\rho, \rho + \e^*],$ and $[\rho+\e^*, 1]$. Using the naive bound of $1$ for the first two integrals, we have
\[
\Prob_{x,\cl{D}}(\widehat{f}_{\MDL} \neq \tP(x)) = \rho + \e^* + \int^{1-\rho}_{\e^*} \exp(c_3\left(L^3V^2 + \rho \edits{I^2}n \right)\ln (I +B(N)) - n\e^2/(2\rho + 2\e)) d\e.
\]
To calculate $\e^*,$ let $a := c_3\left(L^3V^2 + \rho \edits{I^2}n \right)\ln (I +B(N)).$ Then it suffices to solve the equation $a = n\e^2/(2\rho + 2\e)$, which simplifies to the quadratic equation $n\e^2 - 2a\e - 2\rho a = 0$. The equation gives two roots 
\[
\frac{2a \pm \sqrt{4a^2 + 8n\rho a}}{2n},
\]
and one of these roots is negative. Thus 
\[
\e^* = \frac{a}{n} + \sqrt{\frac{a^2}{n^2} + \frac{2\rho a}{n}}.
\]
Finally, by concavity we bound $a - n\e^2/(2\rho + 2\e)$ by its tangent line at $\e^*$, so 
\[
\exp(a - n\e^2/(2\rho + 2\e)) \leq \exp(-\frac{n\e^*(\e^*+2\rho)}{2(\e^* + \rho)^2} (\e - \e^*)),
\]
and thus,
\begin{equation}
\label{eq:thm_noise_gen_helper}
\Prob_{x,\cl{D}}(\widehat{f}_{\MDL} \neq \tP(x)) \leq \rho + \e^* + \frac{2}{n}\left(1 + \frac{\rho^2}{\e^*(\e^* + 2\rho)}\right).
\end{equation}
Redefining $a = b_0 + \rho n b_1$ where $b_0 := c_3L^3V^2\ln(I + B(N)), b_1 = c_3 \edits{I^2} \ln(I + B(N))$, we may write
\begin{align*}
\e^* & = \rho b_1 + \frac{b_0}{n} + \sqrt{2\rho^2b_1 + \rho^2b_1^2 + \frac{2\rho b_0b_1+2\rho b_0}{n} + \frac{b_0^2}{n^2}}.
\end{align*}
Using the bounds $\sqrt{x+ \e} - \sqrt{x} \leq \frac{\e}{2\sqrt{x}}, \sqrt{x + y} \leq \sqrt{x} + \sqrt{y}$ and $b_1 \geq \ln(I + B(N))$ we obtain
\begin{align*}
    \e^* & \leq \rho b_1 + \frac{b_0}{n} + \rho \sqrt{2b_1 + b_1^2} + \frac{1}{2\rho\sqrt{2b_1 + b_1^2}}\left(\frac{2\rho b_0 (b_1+1) + b_0^2}{n} \right) \\
    & \leq \rho(2b_1 + \sqrt{2b_1}) + \frac{1}{n}\left(b_0 + \frac{b_0(b_1+1)}{b_1} +\frac{b_0^2}{2b_1\rho}\right) \\
    & = \rho(2b_1 + \sqrt{2b_1}) + O\left( \frac{\rho^{-1}\ln B(N)}{n}\right)
\end{align*}
Finally, noting that $\e^* \geq \rho$ we bound the last term of Eq.~\eqref{eq:thm_noise_gen_helper} and obtain
\[
\Prob_{x,\cl{D}}(\widehat{f}_{\MDL} \neq \tP(x)) = \rho + \rho(2b_1 + \sqrt{2b_1}) +  O\left( \frac{\rho^{-1}\ln B(N)}{n}\right).
\]
This completes the proof.
\end{proof}

\begin{example}[Prime Numbers with Noise]
To exemplify Theorem \ref{thm:generalization with noise}, consider again the prime checking program of the introduction. The program satisfies $L = 11, V = 9, B(N) = N^2, I =1$. Consider a dataset $\cl{D} = (X_i,Y_i)_{i=1}^n$ where $\rho n$ of the labels $Y_i$ are perturbed arbitrarily, so long as they are natural numbers less than $B(N)$. Recall that $X_i$ are drawn iid uniformly from $[N]$. By Theorem \ref{thm:generalization with noise}, the MDL interpolating neural network satisfies
\[
\Prob_{x, \cl{D}}(\widehat{f}_{\MDL}(x) \neq \tP(x)) \leq \rho (1 + 8c_3\ln N) +  O\left( \frac{\rho^{-1}\ln N}{n}\right).
\]
Depending on $\rho, n, N$ the averaged error can be less than $\frac{1}{\ln N}$, which signals non-trivial generalization better than the prime number theorem. For example, if $n = \sqrt{N}$ and $\rho = 1/(8c_3(\ln N)^3),$ then by the bound above we have
\[
\Prob_{x, \cl{D}}(\widehat{f}_{\MDL}(x) \neq \tP(x)) \leq O\left( \frac{1}{(\ln N)^2} + \frac{(\ln N)^4}{\sqrt{N}} \right) \ll O\left( \frac{1}{\ln N}\right)
\]
which is smaller then the fraction of primes less than $N$, suggesting that the MDL interpolator does not trivially return the zero function.

\end{example}

The proof strategy of previous theorem extends to the case where the corruption pattern of the data is low complexity. One could prove a direct analog of Theorem \ref{thm:generalization with noise}, using an extension of Lemma \ref{lemma:noisy_perturbation_nn}, as long as the noise can be interpolated by a neural network of low description length. Because the description length of the perturbed neural network does not depend on the number of noisy inputs, the generalization result would be stronger than that of Theorem \ref{thm:generalization with noise}. For example, one can consider data which is corrupted by pseudorandom noise. Simple pseudorandom number generators such as the \texttt{xorshift} algorithm \cite{marsaglia2003xorshift} can be described by short programs; some of these can be expressed as SNPs. As this model of noise is different than that typically considered in measure-theoretic probability, we will not elaborate further.

\begin{remark}[Comparison to \cite{harel2024provable}]
\label{remark:harel_comparison}
As previously mentioned, \cite{harel2024provable} consider generalization for minimum size binary threshold neural networks interpolators, which we will denote $\widehat{f}_{\textsf{min-NN}}$, on datasets generated by a true neural network $f_{\textsf{true}}$ of a similar form. The outputs are binary and are corrupted with some probability $\rho$. Although the setting is different from ours, it is instructive to compare the forms of the tempered overfitting bounds. To facilitate comparison with our results, we focus on a subset of their results where $\widehat{f}_{\textsf{min-NN}}$ is trained on a noisy dataset with corruption rate $\rho$, and evaluated on a test point with a noiseless label; this is addressed in Theorem 4.2, Lemma A.9, and also Figure 1b of \cite{harel2024provable}. With arbitrary dependence of the noise and the inputs, \cite{harel2024provable} proves that
\begin{equation}
\label{eq:harel_worst_case}
\Prob\left(\widehat{f}_{\textsf{min-NN}}(x) \neq f_{\textsf{true}}(x) \right) = \frac{1 - \rho^\rho (1-\rho)^{1-\rho} - \rho}{1 - 2\rho} + o_n(1).
\end{equation}
In the case where the labels are independent of the data, it is shown that the error rate is $\rho + o_n(1)$. The results are obtained using information-theoretic techniques and an interesting novel construction using binary threshold neural networks to fit binary label noise. In comparison, Theorem \ref{thm:generalization with noise} proves the worst case error in our setting is $O(\rho) + o_n(1)$. For small corruption rates $\rho$, the error is smaller than in Eq.~ \eqref{eq:harel_worst_case}, which behaves like $\rho \ln 1/\rho + o_n(1).$ The approach for Theorem~ \ref{thm:generalization with noise} does not suggest a way to meaningfully take advantage of an independence assumption for the noise, whereas the information-theoretic approach in~ \cite{harel2024provable} allows for better results in the binary setting. In addition our constants could be improved with a tighter analysis or a smaller alphabet size for the notion of description length. We leave such improvements to future work.
\end{remark}

\section{Discussion}

Theorem \ref{thm:generalization_bound} provides no practical guidance on how to find the minimum description length neural network interpolating the data, beyond brute-force search. Notice that the architecture may change. \cite{lan2022minimum} give very interesting empirical results for a type of MDL network different from ours; they show genetic algorithms are useful for finding the MDL network. Our theorem also does not say anything about neural networks trained with gradient-based methods. Motivated by recent results \cite{mingard2023deep, mingard2019neural, goldblum2023no} outlined in Section \ref{sec:litreview}, proving a result that neural networks optimized through gradient-descent type methods are typically of low complexity could give practical generalization bounds. 



\paragraph{Limitations} The notion of SNPs is somewhat restricted. Although it accommodates many interesting examples, notice that the number of variables cannot scale with the inputs. Moreover, arrays and accessing arrays with variable locations is not allowed. Other natural expressions are disallowed, such as while loops. Furthermore, all variables must be positive integers, and must be bounded by an absolute constant $B := B(N)$. The way Theorem \ref{thm:generalization_bound} depends on $B$ precludes SNPs that do an exponential amount of computation in $N$. The choice of the ReLU function allowed us to encode programmatic statements as neural networks in a direct way, as described in Section \ref{sec:defining_snps}. Some of the constructions in Section \ref{sec:defining_snps} use special properties of the ReLU function, although we expect similar constructions to hold with the threshold activation function $\mathbf{1}\set{x > 0}.$ With different smooth activation functions like the sigmoid, the translation between networks and programs would be more complicated. If one could approximate the ReLU function or threshold units, results analogous to Proposition \ref{prop:nn_compressed_representation}, \ref{prop:efficient_conversion} should hold. 

Many of these limitations can be overcome by increasing the expressivity of SNPs as a programming language, while considering more expressive description measures. As long as there is a conversion between short programs and neural networks of low complexity, the generalization idea of Theorem \ref{thm:generalization_bound} carries through. By extending the programming language, other neural network architectures beyond feedforward networks may have to be considered. For example, can generalization guarantees be obtained for convolutional neural network architectures on structured image data? Can similar guarantees be obtained for recurrent architectures on structured sequence data? In particular, there has been much recent interest in the transformer architecture, in an attempt to explain various phenomena in large language models such as in-context learning, out-of-distribution generalization, and length generalization \cite{wang2024transformers, abbe2024far, ahuja2024provable}. Specializing our argument to transformers and minimum description learning would be of interest.

In some cases, the interpretability of $\widehat{f}_{\MDL}$ can be of interest. This relates to the mechanistic interpretability literature \cite{nanda2023progress}. For example, if the program $\tP$ were unknown, it would be interesting to investigate to what extent $\widehat{f}_{\MDL}$ describes the program $\tP$. Our results do not speak to this; we only provide conversions from simple neural programs to neural networks and not vice-versa. In some examples, we expect $\widehat{f}_{\MDL}$ to be quite different than $\tP.$ Consider Example \ref{ex:primenumber_program}. If the realized training data consisted of all composite numbers, $\widehat{f}_{\MDL}$ would just be the constant function $0$. Outside of special cases like this, the question appears to be difficult.

\section{Acknowledgments}
We thank two anonymous reviewers for helpful comments and suggestions which improved the paper. TS thanks Kevin Guo, Will Hartog, Michael Howes, Andrea Montanari, and Tselil Schramm for useful feedback and discussion. TS acknowledges support from the NSF Graduate Research Fellowship Program under Grant DGE-1656518. SC's research was partially supported by NSF grants DMS-2413864 and DMS-2153654.

\bibliographystyle{alpha}
\bibliography{main}

\newcommand{\etalchar}[1]{$^{#1}$}
\begin{thebibliography}{GMGH{\etalchar{+}}24}

\bibitem[ABAB{\etalchar{+}}21]{abbe2021staircase}
Emmanuel Abbe, Enric Boix-Adsera, Matthew~S Brennan, Guy Bresler, and Dheeraj Nagaraj.
\newblock The staircase property: How hierarchical structure can guide deep learning.
\newblock {\em Advances in Neural Information Processing Systems}, 34:26989--27002, 2021.

\bibitem[ABL{\etalchar{+}}24]{abbe2024far}
Emmanuel Abbe, Samy Bengio, Aryo Lotfi, Colin Sandon, and Omid Saremi.
\newblock How far can transformers reason? the locality barrier and inductive scratchpad.
\newblock {\em arXiv preprint arXiv:2406.06467}, 2024.

\bibitem[ABLR23]{abbe2023generalization}
Emmanuel Abbe, Samy Bengio, Aryo Lotfi, and Kevin Rizk.
\newblock Generalization on the unseen, logic reasoning and degree curriculum.
\newblock {\em arXiv preprint arXiv:2301.13105}, 2023.

\bibitem[AGNZ18]{arora2018stronger}
Sanjeev Arora, Rong Ge, Behnam Neyshabur, and Yi~Zhang.
\newblock Stronger generalization bounds for deep nets via a compression approach.
\newblock In {\em International conference on machine learning}, pages 254--263. PMLR, 2018.

\bibitem[AM24]{ahuja2024provable}
Kartik Ahuja and Amin Mansouri.
\newblock On provable length and compositional generalization.
\newblock {\em arXiv preprint arXiv:2402.04875}, 2024.

\bibitem[Bar93]{barron1993universal}
Andrew~R Barron.
\newblock Universal approximation bounds for superpositions of a sigmoidal function.
\newblock {\em IEEE Transactions on Information theory}, 39(3):930--945, 1993.

\bibitem[BCW{\etalchar{+}}23]{bai2023transformers}
Yu~Bai, Fan Chen, Huan Wang, Caiming Xiong, and Song Mei.
\newblock Transformers as statisticians: Provable in-context learning with in-context algorithm selection.
\newblock {\em arXiv preprint arXiv:2306.04637}, 2023.

\bibitem[BGMSS17]{brutzkus2017sgd}
Alon Brutzkus, Amir Globerson, Eran Malach, and Shai Shalev-Shwartz.
\newblock Sgd learns over-parameterized networks that provably generalize on linearly separable data.
\newblock {\em arXiv preprint arXiv:1710.10174}, 2017.

\bibitem[BGS97]{balcazar1997computational}
Jos{\'e}~L Balc{\'a}zar, Ricard Gavalda, and Hava~T Siegelmann.
\newblock Computational power of neural networks: A characterization in terms of kolmogorov complexity.
\newblock {\em IEEE Transactions on Information Theory}, 43(4):1175--1183, 1997.

\bibitem[BHMM19]{belkin2019reconciling}
Mikhail Belkin, Daniel Hsu, Siyuan Ma, and Soumik Mandal.
\newblock Reconciling modern machine-learning practice and the classical bias--variance trade-off.
\newblock {\em Proceedings of the National Academy of Sciences}, 116(32):15849--15854, 2019.

\bibitem[BLLT20]{bartlett2020benign}
Peter~L Bartlett, Philip~M Long, G{\'a}bor Lugosi, and Alexander Tsigler.
\newblock Benign overfitting in linear regression.
\newblock {\em Proceedings of the National Academy of Sciences}, 117(48):30063--30070, 2020.

\bibitem[BM03]{bartlett2003vapnik}
Peter~L Bartlett and Wolfgang Maass.
\newblock Vapnik-chervonenkis dimension of neural nets.
\newblock {\em The handbook of brain theory and neural networks}, pages 1188--1192, 2003.

\bibitem[BMR{\etalchar{+}}20]{brown2020language}
Tom Brown, Benjamin Mann, Nick Ryder, Melanie Subbiah, Jared~D Kaplan, Prafulla Dhariwal, Arvind Neelakantan, Pranav Shyam, Girish Sastry, Amanda Askell, et~al.
\newblock Language models are few-shot learners.
\newblock {\em Advances in neural information processing systems}, 33:1877--1901, 2020.

\bibitem[BMR21]{bartlett2021deep}
Peter~L Bartlett, Andrea Montanari, and Alexander Rakhlin.
\newblock Deep learning: a statistical viewpoint.
\newblock {\em Acta numerica}, 30:87--201, 2021.

\bibitem[BPKB22]{bhattamishra2022simplicity}
Satwik Bhattamishra, Arkil Patel, Varun Kanade, and Phil Blunsom.
\newblock Simplicity bias in transformers and their ability to learn sparse boolean functions.
\newblock {\em arXiv preprint arXiv:2211.12316}, 2022.

\bibitem[BRY98]{barron1998minimum}
Andrew Barron, Jorma Rissanen, and Bin Yu.
\newblock The minimum description length principle in coding and modeling.
\newblock {\em IEEE transactions on information theory}, 44(6):2743--2760, 1998.

\bibitem[CGM{\etalchar{+}}17]{chen2017recurrent}
Yining Chen, Sorcha Gilroy, Andreas Maletti, Jonathan May, and Kevin Knight.
\newblock Recurrent neural networks as weighted language recognizers.
\newblock {\em arXiv preprint arXiv:1711.05408}, 2017.

\bibitem[CJLZ22]{chen2022nonparametric}
Minshuo Chen, Haoming Jiang, Wenjing Liao, and Tuo Zhao.
\newblock Nonparametric regression on low-dimensional manifolds using deep relu networks: Function approximation and statistical recovery.
\newblock {\em Information and Inference: A Journal of the IMA}, 11(4):1203--1253, 2022.

\bibitem[CTR20]{clark2020transformers}
Peter Clark, Oyvind Tafjord, and Kyle Richardson.
\newblock Transformers as soft reasoners over language.
\newblock {\em arXiv preprint arXiv:2002.05867}, 2020.

\bibitem[Cyb89]{cybenko1989approximation}
George Cybenko.
\newblock Approximation by superpositions of a sigmoidal function.
\newblock {\em Mathematics of control, signals and systems}, 2(4):303--314, 1989.

\bibitem[GFRW23]{goldblum2023no}
Micah Goldblum, Marc Finzi, Keefer Rowan, and Andrew~Gordon Wilson.
\newblock The no free lunch theorem, kolmogorov complexity, and the role of inductive biases in machine learning.
\newblock {\em arXiv preprint arXiv:2304.05366}, 2023.

\bibitem[GMGH{\etalchar{+}}24]{grau2024learning}
Jordi Grau-Moya, Tim Genewein, Marcus Hutter, Laurent Orseau, Gr{\'e}goire Del{\'e}tang, Elliot Catt, Anian Ruoss, Li~Kevin Wenliang, Christopher Mattern, Matthew Aitchison, et~al.
\newblock Learning universal predictors.
\newblock {\em arXiv preprint arXiv:2401.14953}, 2024.

\bibitem[GMKZ20]{goldt2020modeling}
Sebastian Goldt, Marc M{\'e}zard, Florent Krzakala, and Lenka Zdeborov{\'a}.
\newblock Modeling the influence of data structure on learning in neural networks: The hidden manifold model.
\newblock {\em Physical Review X}, 10(4):041044, 2020.

\bibitem[GRS{\etalchar{+}}23]{giannou2023looped}
Angeliki Giannou, Shashank Rajput, Jy-yong Sohn, Kangwook Lee, Jason~D Lee, and Dimitris Papailiopoulos.
\newblock Looped transformers as programmable computers.
\newblock In {\em International Conference on Machine Learning}, pages 11398--11442. PMLR, 2023.

\bibitem[Gr{\"u}07]{grunwald2007minimum}
Peter~D Gr{\"u}nwald.
\newblock {\em The minimum description length principle}.
\newblock MIT press, 2007.

\bibitem[GTLV22]{garg2022can}
Shivam Garg, Dimitris Tsipras, Percy Liang, and Gregory Valiant.
\newblock What can transformers learn in-context.
\newblock {\em A Case Study of Simple Function Classes}, 2022.

\bibitem[HHV{\etalchar{+}}24]{harel2024provable}
Itamar Harel, William Hoza, Gal Vardi, Itay Evron, Nati Srebro, and Daniel Soudry.
\newblock Provable tempered overfitting of minimal nets and typical nets.
\newblock {\em Advances in Neural Information Processing Systems}, 37:53458--53524, 2024.

\bibitem[Hoe94]{hoeffding1994probability}
Wassily Hoeffding.
\newblock Probability inequalities for sums of bounded random variables.
\newblock {\em The collected works of Wassily Hoeffding}, pages 409--426, 1994.

\bibitem[HS97]{hochreiter1997flat}
Sepp Hochreiter and J{\"u}rgen Schmidhuber.
\newblock Flat minima.
\newblock {\em Neural computation}, 9(1):1--42, 1997.

\bibitem[HSW89]{hornik1989multilayer}
Kurt Hornik, Maxwell Stinchcombe, and Halbert White.
\newblock Multilayer feedforward networks are universal approximators.
\newblock {\em Neural networks}, 2(5):359--366, 1989.

\bibitem[HvC]{hinton93keeping}
GE~Hinton and Drew van Camp.
\newblock Keeping neural networks simple by minimising the description length of weights. 1993.
\newblock In {\em Proceedings of COLT-93}, pages 5--13.

\bibitem[KYS23]{kornowski2023tempered}
Guy Kornowski, Gilad Yehudai, and Ohad Shamir.
\newblock From tempered to benign overfitting in relu neural networks.
\newblock {\em Advances in Neural Information Processing Systems}, 36:58011--58046, 2023.

\bibitem[LAG{\etalchar{+}}22]{liu2022transformers}
Bingbin Liu, Jordan~T Ash, Surbhi Goel, Akshay Krishnamurthy, and Cyril Zhang.
\newblock Transformers learn shortcuts to automata.
\newblock {\em arXiv preprint arXiv:2210.10749}, 2022.

\bibitem[Lan09]{landau1909einteilung}
Edmund Landau.
\newblock {\em {\"U}ber die Einteilung der positiven ganzen Zahlen in vier Klassen nach der Mindestzahl der zu ihrer additiven Zusammensetzung erforderlichen Quadrate}.
\newblock 1909.

\bibitem[LBM23]{lin2023transformers}
Licong Lin, Yu~Bai, and Song Mei.
\newblock Transformers as decision makers: Provable in-context reinforcement learning via supervised pretraining.
\newblock {\em arXiv preprint arXiv:2310.08566}, 2023.

\bibitem[LGCK22]{lan2022minimum}
Nur Lan, Michal Geyer, Emmanuel Chemla, and Roni Katzir.
\newblock Minimum description length recurrent neural networks.
\newblock {\em Transactions of the Association for Computational Linguistics}, 10:785--799, 2022.

\bibitem[LKF{\etalchar{+}}24]{lindner2024tracr}
David Lindner, J{\'a}nos Kram{\'a}r, Sebastian Farquhar, Matthew Rahtz, Tom McGrath, and Vladimir Mikulik.
\newblock Tracr: Compiled transformers as a laboratory for interpretability.
\newblock {\em Advances in Neural Information Processing Systems}, 36, 2024.

\bibitem[LL18]{li2018learning}
Yuanzhi Li and Yingyu Liang.
\newblock Learning overparameterized neural networks via stochastic gradient descent on structured data.
\newblock {\em Advances in neural information processing systems}, 31, 2018.

\bibitem[LSSS14]{livni2014computational}
Roi Livni, Shai Shalev-Shwartz, and Ohad Shamir.
\newblock On the computational efficiency of training neural networks.
\newblock {\em Advances in neural information processing systems}, 27, 2014.

\bibitem[LV{\etalchar{+}}08]{li2008introduction}
Ming Li, Paul Vit{\'a}nyi, et~al.
\newblock {\em An introduction to Kolmogorov complexity and its applications}, volume~3.
\newblock Springer, 2008.

\bibitem[Mar03]{marsaglia2003xorshift}
George Marsaglia.
\newblock Xorshift rngs.
\newblock {\em Journal of Statistical software}, 8:1--6, 2003.

\bibitem[M{\'e}z23]{mezard2023spin}
Marc M{\'e}zard.
\newblock Spin glass theory and its new challenge: structured disorder.
\newblock {\em Indian Journal of Physics}, pages 1--12, 2023.

\bibitem[MOKG23]{mali2023computational}
Ankur Mali, Alexander Ororbia, Daniel Kifer, and Lee Giles.
\newblock On the computational complexity and formal hierarchy of second order recurrent neural networks.
\newblock {\em arXiv preprint arXiv:2309.14691}, 2023.

\bibitem[MP43]{mcculloch1943logical}
Warren~S McCulloch and Walter Pitts.
\newblock A logical calculus of the ideas immanent in nervous activity.
\newblock {\em The bulletin of mathematical biophysics}, 5:115--133, 1943.

\bibitem[MRVPL23]{mingard2023deep}
Chris Mingard, Henry Rees, Guillermo Valle-P{\'e}rez, and Ard~A Louis.
\newblock Do deep neural networks have an inbuilt occam's razor?
\newblock {\em arXiv preprint arXiv:2304.06670}, 2023.

\bibitem[MS23]{manoj2023interpolation}
Naren~Sarayu Manoj and Nathan Srebro.
\newblock Interpolation learning with minimum description length.
\newblock {\em arXiv preprint arXiv:2302.07263}, 2023.

\bibitem[MSA{\etalchar{+}}22]{mallinar2022benign}
Neil Mallinar, James Simon, Amirhesam Abedsoltan, Parthe Pandit, Misha Belkin, and Preetum Nakkiran.
\newblock Benign, tempered, or catastrophic: Toward a refined taxonomy of overfitting.
\newblock {\em Advances in neural information processing systems}, 35:1182--1195, 2022.

\bibitem[MSS18]{malach2018provably}
Eran Malach and Shai Shalev-Shwartz.
\newblock A provably correct algorithm for deep learning that actually works.
\newblock {\em arXiv preprint arXiv:1803.09522}, 2018.

\bibitem[MSVP{\etalchar{+}}19]{mingard2019neural}
Chris Mingard, Joar Skalse, Guillermo Valle-P{\'e}rez, David Mart{\'\i}nez-Rubio, Vladimir Mikulik, and Ard~A Louis.
\newblock Neural networks are a priori biased towards boolean functions with low entropy.
\newblock {\em arXiv preprint arXiv:1909.11522}, 2019.

\bibitem[MW23]{mei2023deep}
Song Mei and Yuchen Wu.
\newblock Deep networks as denoising algorithms: Sample-efficient learning of diffusion models in high-dimensional graphical models.
\newblock {\em arXiv preprint arXiv:2309.11420}, 2023.

\bibitem[NCL{\etalchar{+}}23]{nanda2023progress}
Neel Nanda, Lawrence Chan, Tom Lieberum, Jess Smith, and Jacob Steinhardt.
\newblock Progress measures for grokking via mechanistic interpretability.
\newblock {\em arXiv preprint arXiv:2301.05217}, 2023.

\bibitem[NKB{\etalchar{+}}19]{nakkiran2019deep}
Preetum Nakkiran, Gal Kaplun, Yamini Bansal, Tristan Yang, Boaz Barak, and Ilya Sutskever.
\newblock Deep double descent: where bigger models and more data hurt (2019).
\newblock {\em arXiv preprint arXiv:1912.02292}, 6, 2019.

\bibitem[PBM21]{perez2021attention}
Jorge P{\'e}rez, Pablo Barcel{\'o}, and Javier Marinkovic.
\newblock Attention is turing-complete.
\newblock {\em Journal of Machine Learning Research}, 22(75):1--35, 2021.

\bibitem[PMB19]{perez2019turing}
Jorge P{\'e}rez, Javier Marinkovi{\'c}, and Pablo Barcel{\'o}.
\newblock On the turing completeness of modern neural network architectures.
\newblock {\em arXiv preprint arXiv:1901.03429}, 2019.

\bibitem[Raz24]{razin2024understanding}
Noam Razin.
\newblock Understanding deep learning via notions of rank.
\newblock {\em arXiv preprint arXiv:2408.02111}, 2024.

\bibitem[Ris83]{rissanen1983universal}
Jorma Rissanen.
\newblock A universal prior for integers and estimation by minimum description length.
\newblock {\em The Annals of statistics}, 11(2):416--431, 1983.

\bibitem[S{\etalchar{+}}98]{sontag1998vc}
Eduardo~D Sontag et~al.
\newblock Vc dimension of neural networks.
\newblock {\em NATO ASI Series F Computer and Systems Sciences}, 168:69--96, 1998.

\bibitem[SC24]{svete2024transformers}
Anej Svete and Ryan Cotterell.
\newblock Transformers can represent $ n $-gram language models.
\newblock {\em arXiv preprint arXiv:2404.14994}, 2024.

\bibitem[Sch97]{schmidhuber1997discovering}
J{\"u}rgen Schmidhuber.
\newblock Discovering neural nets with low kolmogorov complexity and high generalization capability.
\newblock {\em Neural Networks}, 10(5):857--873, 1997.

\bibitem[SHT23]{sanford2023representational}
Clayton Sanford, Daniel Hsu, and Matus Telgarsky.
\newblock Representational strengths and limitations of transformers.
\newblock {\em arXiv preprint arXiv:2306.02896}, 2023.

\bibitem[SMG24]{stogin2024provably}
John Stogin, Ankur Mali, and C~Lee Giles.
\newblock A provably stable neural network turing machine with finite precision and time.
\newblock {\em Information Sciences}, 658:120034, 2024.

\bibitem[SMW{\etalchar{+}}24]{strobl2024formal}
Lena Strobl, William Merrill, Gail Weiss, David Chiang, and Dana Angluin.
\newblock What formal languages can transformers express? a survey.
\newblock {\em Transactions of the Association for Computational Linguistics}, 12:543--561, 2024.

\bibitem[Sol64]{solomonoff1964formal}
Ray~J Solomonoff.
\newblock A formal theory of inductive inference. part i.
\newblock {\em Information and control}, 7(1):1--22, 1964.

\bibitem[SS92]{siegelmann1992computational}
Hava~T Siegelmann and Eduardo~D Sontag.
\newblock On the computational power of neural nets.
\newblock In {\em Proceedings of the fifth annual workshop on Computational learning theory}, pages 440--449, 1992.

\bibitem[SSBD14]{shalev2014understanding}
Shai Shalev-Shwartz and Shai Ben-David.
\newblock {\em Understanding machine learning: From theory to algorithms}.
\newblock Cambridge university press, 2014.

\bibitem[TB23]{tsigler2023benign}
Alexander Tsigler and Peter~L Bartlett.
\newblock Benign overfitting in ridge regression.
\newblock {\em Journal of Machine Learning Research}, 24(123):1--76, 2023.

\bibitem[TNHA24]{teney2024neural}
Damien Teney, Armand Nicolicioiu, Valentin Hartmann, and Ehsan Abbasnejad.
\newblock Neural redshift: Random networks are not random functions.
\newblock {\em arXiv preprint arXiv:2403.02241}, 2024.

\bibitem[VPCL18]{valle2018deep}
Guillermo Valle-Perez, Chico~Q Camargo, and Ard~A Louis.
\newblock Deep learning generalizes because the parameter-function map is biased towards simple functions.
\newblock {\em arXiv preprint arXiv:1805.08522}, 2018.

\bibitem[VSP{\etalchar{+}}17]{vaswani2017attention}
Ashish Vaswani, Noam Shazeer, Niki Parmar, Jakob Uszkoreit, Llion Jones, Aidan~N Gomez, {\L}ukasz Kaiser, and Illia Polosukhin.
\newblock Attention is all you need.
\newblock {\em Advances in neural information processing systems}, 30, 2017.

\bibitem[WCM22]{wei2022statistically}
Colin Wei, Yining Chen, and Tengyu Ma.
\newblock Statistically meaningful approximation: a case study on approximating turing machines with transformers.
\newblock {\em Advances in Neural Information Processing Systems}, 35:12071--12083, 2022.

\bibitem[WGY21]{weiss2021thinking}
Gail Weiss, Yoav Goldberg, and Eran Yahav.
\newblock Thinking like transformers.
\newblock In {\em International Conference on Machine Learning}, pages 11080--11090. PMLR, 2021.

\bibitem[WWHL24]{wang2024transformers}
Zixuan Wang, Stanley Wei, Daniel Hsu, and Jason~D Lee.
\newblock Transformers provably learn sparse token selection while fully-connected nets cannot.
\newblock {\em arXiv preprint arXiv:2406.06893}, 2024.

\bibitem[YMT{\etalchar{+}}23]{yang2023introduction}
Yibo Yang, Stephan Mandt, Lucas Theis, et~al.
\newblock An introduction to neural data compression.
\newblock {\em Foundations and Trends{\textregistered} in Computer Graphics and Vision}, 15(2):113--200, 2023.

\bibitem[ZBH{\etalchar{+}}21]{zhang2021understanding}
Chiyuan Zhang, Samy Bengio, Moritz Hardt, Benjamin Recht, and Oriol Vinyals.
\newblock Understanding deep learning (still) requires rethinking generalization.
\newblock {\em Communications of the ACM}, 64(3):107--115, 2021.

\bibitem[ZBL{\etalchar{+}}23]{zhou2023algorithms}
Hattie Zhou, Arwen Bradley, Etai Littwin, Noam Razin, Omid Saremi, Josh Susskind, Samy Bengio, and Preetum Nakkiran.
\newblock What algorithms can transformers learn? a study in length generalization.
\newblock {\em arXiv preprint arXiv:2310.16028}, 2023.

\bibitem[ZLH22]{zhu2022nearly}
Huangjun Zhu, Zihao Li, and Masahito Hayashi.
\newblock Nearly tight universal bounds for the binomial tail probabilities.
\newblock {\em arXiv preprint arXiv:2211.01688}, 2022.

\end{thebibliography}

\appendix

\section{Miscellaneous Results}
\label{sec:other_lemmas}

\begin{lemma}
\label{lemma:binomial_tail}
For all $\rho,\e \in [0,1]$ with $\rho + \e \in [0,1]$ we have
\[
\Prob(\textsf{Bin}(n,\rho+\e) \leq n\rho) \leq \exp(-\frac{n\e^2}{2(\rho + \e)}).
\]
\end{lemma}
\begin{proof}
A inequality of Hoeffding \cite{hoeffding1994probability} (see also \cite{zhu2022nearly}) shows that 
\[
\Prob(\textsf{Bin}(n,\rho+\e) \leq n\rho) \leq \exp(-n D(\rho \ || \  \rho+\e)).
\]
Here $D(p_0 \ || \ p_1)$ is relative entropy between two Bernoulli distributions:
\[
D(p_0 \ || \ p_1) = p_0 \ln \frac{p_0}{p_1} + (1-p_0) \ln \frac{1 - p_0}{1 - p_1}.
\]
It suffices to show that $D(\rho \ || \ \rho + \e) \geq \frac{\e^2}{2(\rho+\e)}$. At $\e = 0$, both quantities are zero. We will show that
\[
\frac{d}{d\e} \bigg|_{\e_0} D(\rho \ || \ \rho + \e) \geq \frac{d}{d\e}\bigg |_{\e_0} \frac{\e^2}{2(\rho+\e)}
\]
for all $\e_0$. Computing both derivatives and simplifying, we obtain
\begin{align*}
 -\frac{\rho}{\rho+\e_0} + \frac{1-\rho}{1-\rho-\e_0} & \geq -\frac{\e_0^2}{2(\rho+\e_0)^2} + \frac{\e_0}{\rho + \e_0} \\
\Leftrightarrow  -\rho + \frac{(1-\rho)(\rho + \e_0)}{1-\rho-\e_0} & \geq -\frac{\e_0^2}{2(\rho+\e_0)} + \e_0 \\
\Leftrightarrow  \frac{\e_0}{1 - \rho - \e_0} & \geq \frac{\e_0^2 + 2\rho \e_0}{2(\rho + \e_0)} \\
\Leftrightarrow   2(\rho + \e_0) & \geq (\e_0 + 2\rho)(1 - \rho - \e_0),
\end{align*}
which is equivalent to the trivial inequality $2\rho^2 + (3\rho+1) \e_0 + \e_0^2 \geq 0$.
\end{proof}

\section{Full Simple Neural Program Descriptions of Examples}
\label{sec:full_programs}

\begin{example}[Prime Number Checking]
\label{ex:primenumber_program_full}
Let $N$ be fixed. For any $n \leq N$, checking whether $n$ is a prime number can be expressed as an SNP. 

\begin{minted}{python}
input n
int i = 2
int j = 2
int i_mult = 0
int res_mult = 0
int prod = 0
int t = 0
int sum = 0
bool output = 0
bool prod_equals = 0
\end{minted}
\begin{minted}[linenos]{python}
for i = 2,...,n:
    for j = 2,...,n:
        res_mult = 0
        i_mult = 0
        for i_mult = 1,...,j:
            res_mult = res_mult + i
        prod = res_mult
        prod_equals = (prod == n)
        res = res + prod_equals
output = (res > 0)
return output
\end{minted}
\end{example}

\begin{example}[Sums of Squares]
\label{ex:sos_full}
Consider the sum of squares example from before. It has variable context
\begin{minted}{python}
input n
int i = 0
int j = 0
int res = 0
int idx1 = 1
int idx2 = 1
int square1 = 0
int square2 = 0
bool output = 0
bool sum_of_squares = 0
\end{minted}
with the full program stated as
\begin{minted}[linenos]{python}
for i = 0,...,n:
    for j = 0,...,n:
        square1 = 0
        for idx1 = 1,...,i:
            square1 = square1 + i
        square2 = 0
        for idx2 = 1,...,j:
            square2 = square2 + j
        sum = square1 + square2
        sum_of_squares = (sum == n)
        res = res + sum_of_squares
output = (res > 0)
return output
\end{minted}
\end{example}

\end{document}